\documentclass[hidelinks,onefignum,onetabnum]{siamart220329}

\usepackage{lipsum}
\usepackage{amsfonts}
\usepackage{graphicx}
\usepackage{epstopdf}
\usepackage{algorithmic}
\ifpdf
  \DeclareGraphicsExtensions{.eps,.pdf,.png,.jpg}
\else
  \DeclareGraphicsExtensions{.eps}
\fi

\newsiamremark{remark}{Remark}
\newsiamremark{hypothesis}{Hypothesis}
\crefname{hypothesis}{Hypothesis}{Hypotheses}
\newsiamthm{claim}{Claim}

\usepackage{microtype}
\usepackage{graphicx}
\usepackage{subfigure}
\usepackage{booktabs}

\headers{Learning Pseudo-Contractive Denoisers for Inverse Problems}{Wei deliang, Chen Peng, and Li Fang}

\title{Learning Pseudo-Contractive Denoisers for Inverse Problems\thanks{Submitted to the editors DATE.
\funding{This work is supported in part by Natural Science Foundation of Shanghai (No. 22ZR1419500), Science and Technology Commission of Shanghai Municipality (No. 22DZ2229014), and the Open Project of Shanghai Key Laboratory of Magnetic Resonance, ECNU.}}}

\author{Wei Deliang\thanks{School of Mathematical Sciences, Ministry of Education Key Laboratory of Mathematics and Engineering Applications and Shanghai Key Laboratory of PMMP, East China Normal University, Shanghai 200241, China (\email{52215500006@stu.ecnu.edu.cn},\email{52265500005@stu.ecnu.edu.cn},\email{fli@math.ecnu.edu.cn}). Corresponding author: Fang Li. }
\and Chen Peng\footnotemark[2]
\and Li Fang\footnotemark[2]}

\begin{document}

\maketitle

\begin{abstract}
Deep denoisers have shown excellent performance in solving inverse problems in signal and image processing. In order to guarantee the convergence, the denoiser needs to satisfy some Lipschitz conditions like non-expansiveness. However, enforcing such constraints inevitably compromises recovery performance. This paper introduces a novel training strategy that enforces a weaker constraint on the deep denoiser called pseudo-contractiveness. By studying the spectrum of the Jacobian matrix, relationships between different denoiser assumptions are revealed. Effective algorithms based on gradient descent and Ishikawa process are derived, and further assumptions of strict pseudo-contractiveness yield efficient algorithms using half-quadratic splitting and forward-backward splitting. The proposed algorithms theoretically converge strongly to a fixed point. A training strategy based on holomorphic transformation and functional calculi is proposed to enforce the pseudo-contractive denoiser assumption. Extensive experiments demonstrate superior performance of the pseudo-contractive denoiser compared to related denoisers. The proposed methods are competitive in terms of visual effects and quantitative values.
\end{abstract}



\section{Introduction}

Inverse problems aim to recover the potential signal from down sampled or corrupted obsevations. A typical inverse problem takes form of:
\begin{equation}\label{original inverse problem}
    f = Ku + n,
\end{equation}
where $f$ is the observed signal, $u$ is the potential signal, $K$ is the degradation operator, and $n$ is the noise following certain distributions. Different values of $K$ and $n$ correspond to different missions including denoising, deblurring, inpainting, super-resolution, and medical imaging. In order to recover $u$ from $f$, a variational approach is considered:
\begin{equation}\label{model}
    \hat{u}=\arg\min\limits_{u\in V}F(u)+ G(u;f),
\end{equation}
where $V$ is the Hilbert space, $F$ denotes the prior regularization term, and $G$ is the data fidelity term. Typical choices for $F$ include total variation \cite{sauer1992bayesian, rudin1992nonlinear} and its extensions \cite{kindermann2005deblurring,bredies2010total}, weighted nuclear norm \cite{gu2014weighted}, group-based low rank prior \cite{groupsparsity} et al.. First order methods are employed to solve (\ref{model}), such as the alternating direction method with multipliers (ADMM) \cite{boyd2011distributed}:
\begin{equation}\label{admm}
    \begin{array}{ll}
    u^{k+1}&=Prox_{\frac{F}{\beta}}(v^k-b^k),\vspace{0.5ex}\\
    v^{k+1}&=Prox_{\frac{G}{\beta}}(u^{k+1}+b^k),\vspace{0.5ex}\\
    b^{k+1}&=b^k+u^{k+1}-v^{k+1},
    \end{array}
\end{equation}
where $\beta>0$. For a given proper, closed, and convex function $F:V\rightarrow (-\infty,\infty]$, the proximal operator $Prox_F:V\rightarrow V$ is defined as:
\begin{equation}
    Prox_F(y)=\arg\min\limits_{x\in V}F(x)+\frac{1}{2}\|x-y\|^2.
\end{equation}
Noticing that $Prox_{\frac{F}{\beta}}(\cdot)$ is a Gaussian denoiser, in \cite{venkatakrishnan2013plug}, Venkatakrishnan et al. proposed to replace the $u$-subproblem in (\ref{admm}) with arbitrary Gaussian denoiser $D_\beta$ in a plug-and-play (PnP) fashion, and arrived at PnP-ADMM:
\begin{equation}
    \begin{array}{ll}
         u^{k+1}&=D_\beta(v^k-b^k), \vspace{0.5ex}\\
         v^{k+1}&=Prox_{\frac{G}{\beta}}(u^{k+1}+b^k),\vspace{0.5ex}\\
         b^{k+1}&=b^k+u^{k+1}-v^{k+1}.
    \end{array}
\end{equation}
Here, $D_\beta$ is a Gaussian denoiser with denoise strength $\beta$. When $\beta$ gets bigger, the denoise strength gets smaller. 

Interestingly, PnP-ADMM, along with other PnP methods, has demonstrated remarkable recovery effects in a diverse range of areas, such as bright field electron tomography \cite{sreehari2016plug}, camera image processing \cite{heide2014flexisp}, low-dose CT imaging \cite{venkatakrishnan2013plug,peng2023denoising}, image denoising \cite{le2023preconditioned}, deblurring \cite{laroche2023provably}, inpainting \cite{zhu2023denoising}, and super-resolution \cite{laroche2023deep}. Nevertheless, it is difficult to analyse the convergence, since $D_\beta$ is a black box. How to guarantee the convergence of PnP algorithms with weaker assumptions and more powerful denoisers has become a challenging research topic. \par

The existing approaches to guarantee the convergence of PnP methods can be classified into two categories.\par

The first class aims to find a potential function $F:V\rightarrow (-\infty,\infty]$, such that $D_\beta=\nabla F$ or $D_\beta=Prox_F$. In \cite{sreehari2016plug}, by studying the Jacobian matrix $J(x)=\nabla D_\beta(x)$, Sreehari et al. first proved that when $J$ is symmetric with eigenvalues in $[0,1]$ for any $x\in V$, there exists some proper, closed, and convex $F$, such that $D_\beta(\cdot)=Prox_{\frac{F}{\beta}}(\cdot)$ is indeed a proximal operator. However, this assumption may be too strong, that most denoisers like non-local means (NLM) \cite{buades2005review}, BM3D \cite{dabov2006image}, DnCNN \cite{zhang2017beyond}, and UNet \cite{ronneberger2015u} violate it. In \cite{sreehari2016plug}, Sreehari et al. proposed symmetric NLM and plugged it into the PnP framework, but the reconstruction seems not satisfactory.  
Romano et al. proposed the regularization by denoising (RED) method, which is more flexible than PnP-ADMM  \cite{romano2017little}. The RED prior term takes the form of $F(x) = \frac{1}{2}\langle x, x-D_\beta(x)\rangle$. Romano et al. proved that when $D_\beta$ is locally homogeneous, $\nabla D_\beta$ is symmetric with spectral radius less than one, one has $\nabla F(x)=x-D_\beta(x)$, and that PnP-GD and PnP-ADMM with RED prior converge. Yet the assumptions are impractical. As reported by Reehorst and Schniter, deep denoisers do not satisfy these assumptions \cite{reehorst2018regularization}. In \cite{cohen2021has}, instead of training a Gaussian denoiser $D_\beta$, Cohen et al. parameterized an implicit convex function $F:V\rightarrow (-\infty,\infty)$ with a neural network by enforcing non-negative weights and convex, non-decreasing activations, such that $F$ is convex, and  $D_\beta(\cdot)=\nabla F(\cdot):V\rightarrow V$ outputs clean images. By doing so, an implicit convex prior $F$ is obtained, and a convergent algorithm based on gradient decent (GD) is derived. Unfortunately, experimental results show that a convex regularization term has limited recovery effects and slow convergence. In \cite{hurault2022gradient}, Hurault et al. proposed the gradient step (GS) denoiser $D_\beta=I-\nabla F$, where $F$ is parameterized by DRUNet \cite{zhang2021plug}. In \cite{hurault2022proximal}, Hurault et al. proposed the proximal DRUNet (Prox-DRUNet), which requires that $\nabla F$ is $L$-Lipschitz with $L\le 0.5$, $F$ is bounded from below, $G$ is proper, closed, convex, and $F$ verifies Kurdyka-Lojasiewicz (KL) property \cite{attouch2010proximal,frankel2015splitting}, and that the iterations are bounded. Under these assumptions, Hurault et al. proved the convergence of PnP with half-quadratic splitting (PnP-HQS) and PnP-ADMM. Experimental results showed that it is possible to learn a gradient step denoiser while not greatly compromising the denoising performance. Nonetheless, the assumptions may still be too strong: the constraint on $L\le 0.5$ limited the denoising performance, see \cite{hurault2022proximal}.
Although the training strategy used by Hurault et al. guarantees $L\le 0.5$, it is difficult to verify the lower boundedness and KL property of $F$, as well as the boundedness of the iterations. \par

The second class investigates under what assumptions of $D_\beta$ does PnP has a fixed-point convergence. In the work of Sreehari et al. \cite{sreehari2016plug}, the Jacobian matrix $J$ of the denoiser $D_\beta$ is assumed to be symmetric, with eigenvalues lying inside $[0,1]$. Then $D_\beta$ is firmly non-expansive. As a result, PnP-ADMM converges to a fixed point. Inspired by this pioneer work, Chan et al. analyzed convergence with a bounded denoiser assumption \cite{chan2016plug}. The denoising strength decreases to ensure the convergence. In \cite{buzzard2018plug}, Buzzard et al. explained via the framework of consensus equilibrium. The convergence of PnP is proved for nonexpansive denoisers. In \cite{sun2019online}, Sun et al. analyzed the convergence of PnP with proximal gradient descent (PnP-PGM) under the assumption that $D_\beta$ is $\theta$-averaged ($\theta\in(0,1)$). The averagedness assumption is too restricted, since many denoisers cannot be considered as averaged denoiser \cite{laumont2023maximum}. In \cite{ryu2019plug}, Ryu et al. assumed the contractiveness of $I-D_\beta$. They studied the convergence of PnP-ADMM and PnP with forward-backward splitting (PnP-FBS, which is equivalent to PnP-PGM). To ensure the contractiveness of $I-D_\beta$, real spectral normalization (RealSN) was proposed, which normalized the spectral norm of each layer. However, RealSN is time consuming, and is designed specifically for denoisers with cascade residual learning structures like DnCNN, and thus is not suitable for other networks like UNet or Transformer. Besides, the contractive constraint of the residual part $I-D_\beta$ seems to limit the performance of PnP-ADMM and PnP-FBS, see \cite{ryu2019plug}. In \cite{cohen2021regularization}, Cohen et al. reformulated RED as a convex minimization problem utilizing a projection (RED-PRO) onto the fixed-point set of demi-contractive denoisers. RED-PRO gives strong link between PnP and RED framework. The denoiser in RED-PRO is assumed to be demi-contractive, locally homogeneous, with symmetric Jacobian matrix with spectral radius less than one. Similarly to RED, the assumptions are too good to be true in practice, and difficult to validate. In \cite{pesquet2021learning}, Pesquet proved the convergence of PnP-FBS when $2D_\beta-I$ is non-expansive. In this case, $D_\beta$ is a resolvent of some maximally monotone operator. They proposed an efficient training method to ensure this condition. Numerical results indicate the effectivenees of this training method. Since that the non-expansiveness of $D_\beta$ can be drawn from the non-expansiveness of $2D_\beta-I$, the constraint is more restrictive, and the performance of the denoiser is not satisfying.
\par
\textbf{Contributions.} As discussed above, in order to guarantee the convergence of PnP and RED algorithms, the previous works assume the Lipschitz properties of the denoisers. However, enforcing such assumptions inevitably comprises the denoising performance. To address these issues, in this paper, we propose convergent plug-and-play methods with pseudo contractive denoisers. Overall, our main contributions are threefold:
\begin{itemize}
    \item The assumption regarding the denoiser is pseudo-contractiveness, which is weaker than that of existing methods. To ensure this assumption, an effective training strategy has been proposed.
    \item Convergent plug-and-play Ishikawa methods based on GD, HQS, and FBS are proposed.
    \item Numerical experiments show that the proposed methods are competitive compared with other closely related methods in terms of visual effects, and quantitive values.
\end{itemize}

The rest of this paper is organized as follows. In Section \ref{sec 2} we introduce the pseudo-contractive denoisers, and study the spectrum of the Jacobian under different assumptions. In Section \ref{sec 3}, we propose the algorithms based on Ishikawa process, and analyze their convergence. In Section \ref{sec 4}, we present an effective training strategy to ensure the assumption. In Section \ref{sec 5}, we present some experimental results. Finally, we conclude the paper in Section \ref{sec 6}.

\section{Pseudo-contractive Denoisers}\label{sec 2}
In order to guarantee the convergence of PnP methods, many Lipschitz assumptions have been made on the denoiser $D$. In this section, we briefly review some closely related assumptions, and introduce the pseudo-contractive denoisers.\par
Let $V$ be the real Hilbert space, $\langle\cdot,\cdot\rangle$ be the inner product on $V$, and $\|\cdot\|$ be the induced norm. 
\begin{itemize}
    \item Non-expansive $D$:
    \begin{equation}
        \|D(x)-D(y)\|\le \|x-y\|,\forall x,y\in V.
    \end{equation}
    \item $\theta$-averaged $D$ ($\theta\in(0,1]$):
    \begin{equation}
        \left\| \left[\left(1-\frac{1}{\theta}\right)I+\frac{1}{\theta}D\right](x)-\left[\left(1-\frac{1}{\theta}\right)I+\frac{1}{\theta}D\right](y) \right\|\le \|x-y\|, \forall x,y\in V.
    \end{equation}
    \item Contractive $I-D$ ($r<1$):
    \begin{equation}
        \|(I-D)(x)-(I-D)(y)\|\le r\|x-y\|, \forall x,y\in V.
    \end{equation}
\end{itemize}
$\theta$-averaged $D$ can be written as 
\begin{equation}\label{averaged temp}
    D = \theta N + (1-\theta)I, 
\end{equation}
where $N$ is a non-expansive mapping. Averaged mappings are non-expansive. Firmly non-expansiveness is a special case of averagedness with $\theta=\frac{1}{2}$. \par
A mapping $D$ is said to be pseudo-contractive \cite{rafiq2007mann, weng1991fixed,hicks1977mann}, if there exists $k\le 1$, such that 
\begin{equation}\label{tmp 1245}
    \|D(x)-D(y)\|^2\le \|x-y\|^2+k\|(I-D)(x)-(I-D)(y)\|^2,\forall x,y\in V.
\end{equation}
When $0\le k<1$, $D$ is strictly pseudo-contractive \cite{chidume1987iterative,weng1991fixed}. Non-expansiveness is a special case of pseudo-contractive with $k=0$.\par

Lemma \ref{lemma 1} gives an equivalent definition of $k$-strictly pseudo-contractive mappings, and therefore gives the relationship between strictly pseudo-contractive mappings and the averaged mappings in the form of (\ref{averaged temp}).
\begin{lemma}\label{lemma 1}(proof in Appendix \ref{appendix lemma1})
The following two statements are equivalent:
\begin{itemize}
    \item $D:V\rightarrow V$ is $k$-strictly pseudo-contractive with $k\in[0,1)$;
    \item $D:V\rightarrow V$ can be written as
    \begin{equation}
        D=\frac{1}{1-k}N-\frac{k}{1-k}I,
    \end{equation}
    where $N:V\rightarrow V$ is non-expansive.
\end{itemize}
\end{lemma}

It is worth noting that $D$ is pseudo-contractive, if and only if $I-D$ is monotone:
\begin{lemma}\label{lemma 1.5}(proof in Appendix \ref{appendix lemma1.5})
Let $D:V\rightarrow V$ be a mapping in the real Hibert space $V$. Then, $D$ is pseudo-contractive, if and only if $I-D$ is monotone, that is
\begin{equation}\label{monotone I-D}
    \langle (I-D)(x)-(I-D)(y),x-y\rangle \ge 0.
\end{equation}
\end{lemma}

We illustrate the relationships between these properties below:
\begin{equation}
    \text{Firmly Non-expansive}\Rightarrow \text{Averaged}\Rightarrow\text{Non-expansive}\Rightarrow\text{Pseudo-contractive}.
\end{equation}
An intuitive illustration is given in Fig. \ref{spectral fig}.

It has been reported in \cite{chan2016plug} that native off-of-the-shelf denoisers are not non-expansive. Besides, imposing non-expansive denoiser alters its denoising performance \cite{hurault2022proximal}. Pseudo-contractiveness enlarge the range of the denoisers in the following sense. We suppose that $D$ is a deep Gaussian denoiser, which inputs a noisy image and outputs a clean image. In this setting, $I-D$ outputs the predicted noise. Pseudo-contractive $D$ means that the difference between two output clean images is smaller than the sum of the difference between the input noisy images and the difference between the predicted noises. As a result, Pseudo-contractiveness is a weaker assumption on the deep denoisers than non-expansiveness, averagedness, and firmly non-expansiveness. \par

We further explore the potential relationships between different assumptions on the denoisers by studying the spectrum distribution.
Let $D\in \mathcal{C}^1[V]$, and $J(x) = \nabla_x D$ be the Jacobian matrix at point $x\in V$ of $D$. By the mean value theorem, (\ref{monotone I-D}) can be rewritten as
\begin{equation}
    \langle (I-J^\mathrm{T}(\xi))(x-y),x-y\rangle \ge 0, \xi = \xi(x,y)\in V.
\end{equation}
Thus $D$ is pseudo-contractive, if there holds 
\begin{equation}\label{reg PC}
    \langle (I-J^\mathrm{T})(x-y),x-y\rangle \ge 0,
\end{equation}
for any $x,y,\xi\in V$, $J=J(\xi)$. We refer (\ref{reg PC}) to \emph{the regularity condition} of pseudo-contractiveness. We decompose $J$ by its symmetric part $S=\frac{1}{2}(J+J^\mathrm{T})$ and its anti-symmetric part $A=\frac{1}{2}(J-J^\mathrm{T})$. For any $x,y\in V$, we have 
\begin{equation}
    \langle (I-J^\mathrm{T})(x-y),x-y\rangle
    = \langle (I-S)(x-y),x-y \rangle + \langle A^\mathrm{T}(x-y),x-y\rangle
    =\langle (I-S)(x-y),x-y \rangle.
\end{equation}

As a result, condition (\ref{reg PC}) is equivalent to that any eigenvalue of $S$ is not larger than $1$. Thus, the real part of any eigenvalue of $J$ is smaller than $1$. That is, the eigenvalue of $J$ for pseudo-contractive $D$ lies inside the half plane, $\sigma(J)\subset\{z\in \mathbb{C}: real(z)\le 1\}$, where $\sigma(J)$ denotes the spectrum set of $J$ as follows:
\begin{equation}\label{sigma J}
\sigma(J)=\mathop{\bigcup}\limits_{x\in V}\sigma(J(x)).
\end{equation}

Let $\|\cdot\|_*$ denote the spectral norm. Similarly, we give the \textit{regularity conditions} for the Jacobian $J$ under different assumptions on the denoiser $D$, as well as the distribution regions $\sigma(J)$ in (\ref{111})-(\ref{rc pc}). Note that these regularity conditions are sufficient conditions for a denoiser to satisfy the corresponding assumptions.
\begin{itemize}
    \item Non-expansive $D$:
    \begin{equation}\label{111}
        \|J\|_*\le 1, \ \sigma(J)\subset\{z\in \mathbb{C}: |z|\le 1\}.
    \end{equation}
    \item $\theta$-averaged $D$ ($\theta\in(0,1]$):
    \begin{equation}
        \left\| \left[\left(1-\frac{1}{\theta}\right)I+\frac{1}{\theta}J\right]\right\|_*\le 1, \ \sigma(J)\subset\left\{z\in\mathbb{C}: \left|1-\frac{1}{\theta}+\frac{1}{\theta}z\right|\le 1\right\}.
    \end{equation}
    \item Contractive $I-D$ ($r<1$):
    \begin{equation}
        \|I-J\|_*\le 1, \ \sigma(J)\subset\{z\in\mathbb{C}:|z-1|\le r \}.
    \end{equation}
    \item $k$-strictly pseudo-contractive $D$ ($k<1$):
    \begin{equation}\label{rc spc}
        \|kI+(1-k)J\|_*\le 1, \ \sigma(J)\subset\{z\in\mathbb{C}:|(1-k)z+k|\le 1 \}.
    \end{equation}
    \item Pseudo-contractive $D$:
    \begin{equation}\label{rc pc}
        \langle (I-J^\mathrm{T})(x-y),x-y\rangle \ge 0, \ \sigma(J)\subset\{z\in\mathbb{C}:real(z)\le 1 \}.
    \end{equation}
\end{itemize}

\begin{figure}[htbp]
  \centering
  \includegraphics[width=10cm]{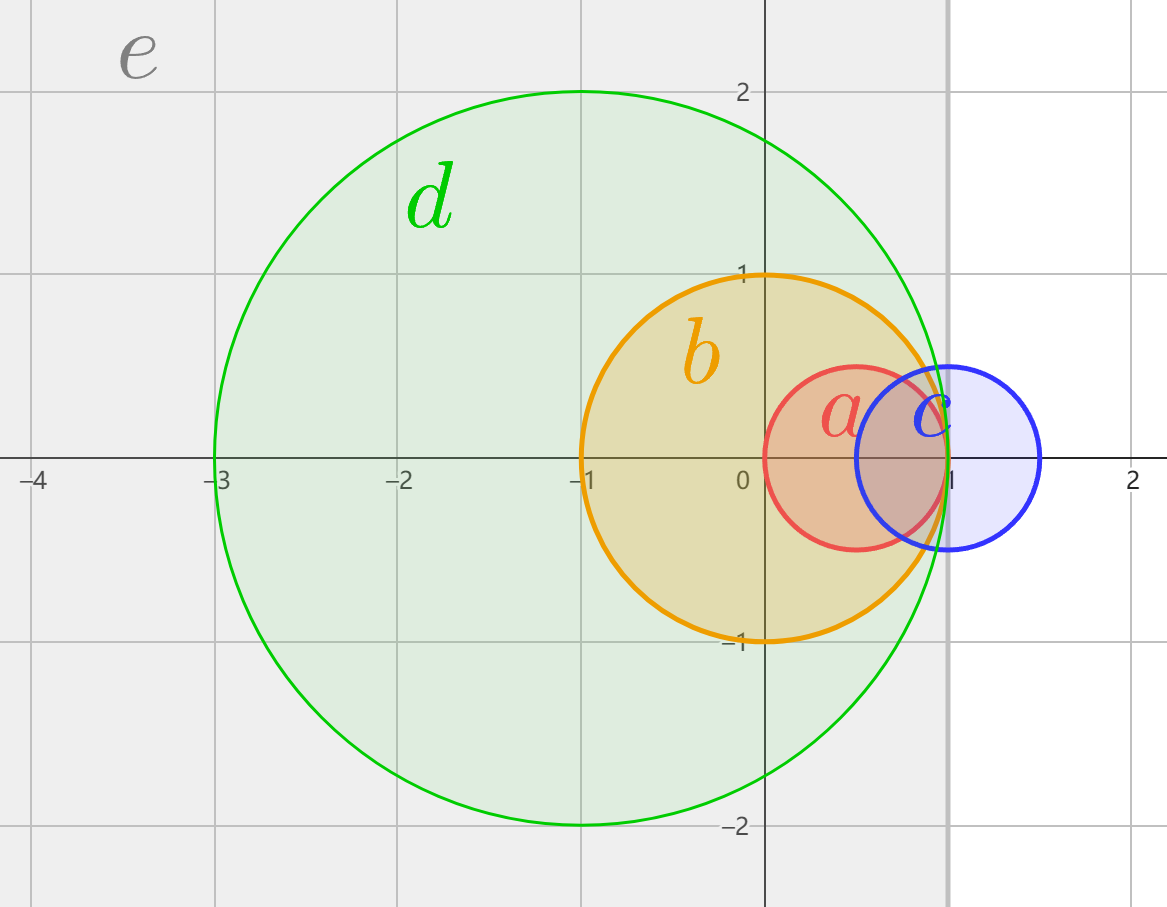}
  \caption{Spectrum distributions on the complex plane for the Jacobian under different assumptions. (a) Firmly non-expansiveness, that is $\frac{1}{2}$-averagedness, $\sigma(J)\subset\{z\in\mathbb{C}: |2z-1|\le 1\}$; (b) Non-expansiveness, $\sigma(J)\subset\{z\in\mathbb{C}: |z|\le 1\}$; (c) Contractiveness of $I-J$ with $r=\frac{1}{2}$, $\sigma(J)\subset\{z\in\mathbb{C}:|z-1|\le\frac{1}{2} \}$; (d) $k$-strictly pseudo-contractiveness with $k=\frac{1}{2}$, $\sigma(J)\subset\{z\in\mathbb{C}:|z+1|\le 2\}$; (e) Pseudo-contractiveness, $\sigma(J)\subset\{z\in\mathbb{C}:real(z)\le 1\}$.
}\label{spectral fig}
\end{figure}

In Fig. \ref{spectral fig}, we present an intuitive illustration of the relationships between different assumptions on $D$. We display the regions of spectrum distribution on the complex plane. In (a), we depict the region when $D$ is $\frac{1}{2}$-averaged, which corresponds to firm non-expansiveness. (a) is contained within panel (b), representing the unit disk when $D$ is non-expansive. This reveals that firm non-expansiveness is a specific instance of non-expansiveness. (c) showcases the region where $I-D$ is contractive with $r=\frac{1}{2}$. In Figs. \ref{spectral fig} (d) and (e), we plot the distribution region for $\frac{1}{2}$-strictly pseudo-contractive and pseudo-contractive $D$, respectively. The are in (d) encompasses non-expansiveness and is enclosed by the half-plane in (e). This suggests that the (strictly) pseudo-contractive property constitutes a significantly weaker assumption for denoisers.

\section{The proposed algorithms}\label{sec 3}
In this section, we propose PnP-based algorithms to solve (\ref{model}) using (strictly) pseudo-contractive denoisers. Before that, we briefly review three existing PnP methods. \par

Gradient descent (GD) method solves (\ref{model}) by 
\begin{equation}
    u^{n+1} = u^n - \alpha (\nabla F+\nabla G)u^n,
\end{equation}
where $\alpha>0$ is the step size. When $F$ is parameterized by a neural network $D_\beta$, $\nabla F$ is often replaced by $ I-D_\beta$ \cite{romano2017little}. Then, PnP-GD takes the form of 
\begin{equation}\label{PnP-GD}
     u^{n+1} = [(1-\alpha)I+\alpha T]u^n , T=D_\beta-\nabla G.\vspace{-1ex}
\end{equation}

Unlike PnP-GD, many PnP methods can be written as the composition of two mappings. For example, the iterations of PnP-HQS and PnP-FBS to solve (\ref{model}) takes the form of 
\begin{equation}\label{PnP-HQS}
\begin{array}{ll}
\text{PnP-HQS:}     & u^{n+1} = T(u^n), \ T = D_\beta\circ Prox_{\frac{G}{\beta}},  \\
  \text{PnP-FBS:}   & u^{n+1} = T(u^n), \ T = D_\beta\circ(I-\lambda \nabla G).
\end{array}
\end{equation}

$T$ in PnP-GD is the sum of $D_\beta$ and $-\nabla G$, while $T$ in PnP-HQS and PnP-FBS is composed of two mappings. When $D_\beta$ is assumed to be pseudo-contractive, it is necessary to study the property of $T$ in these cases. Lemma \ref{lemma 6} gives the condition that the sum of $D_\beta$ and $-\nabla G$ is pseudo-contractive.
\begin{lemma}\label{lemma 6}(proof in Appendix \ref{appendix lemma6})
Let $D$ be pseudo-contractive, and $G$ be proper, closed, and convex. Then $T=D-\nabla G$ is also pseudo-contractive.
\end{lemma}

Lemma \ref{lemma 4} gives the condition that a strictly pseudo-contractive mapping composed with an averaged mapping is still pseudo-contractive. 

\begin{lemma}\label{lemma 4}(proof in Appendix \ref{appendix lemma4})
Let $D$ be $k$-strictly pseudo-contractive, and $P$ be $\theta$-averaged, $k,\theta\in(0,1].$ If $k< 1-\theta$, the composite operator $D\circ P$ is $l$-strictly pseudo-contractive, where 
\begin{equation}
    0\le l = \frac{k(1-\theta)}{(1-\theta)-k\theta}<1.
\end{equation}
If $k=1-\theta$, $D\circ P$ is pseudo-contractive. Besides, when $k<1$, $D\circ P$ is $\frac{1+k}{1-k}$-Lipschitz.
\end{lemma}
By Lemmas \ref{lemma 6}-\ref{lemma 4}, when $D_\beta$ is (strictly) pseudo-contractive, $T$ is pseudo-contractive. We need a special iteration schemes to find the fixed point of a pseudo-contractive mapping $T$. 

Ishikawa proposed the following process to find the fixed point of a Lipschitz pseudo-contractive mapping $T$ over a compact convex set $K$ \cite{ishikawa1974fixed}. He proved the following theorem.

\begin{theorem}\label{thm 0.5}
Let $K$ be a compact convex subset of a Hilbert space $V$, $T:K\rightarrow K$ is a Lipschitz and pseudo-contractive mapping, and $x^0\in K$, then the sequence $\{x^n\}$ converges strongly to a fixed point of $T$, where $x^n$ is defined iteratively for $n\ge 0$ by:
\begin{equation}\label{ishkawa process}
    \begin{array}{ll}
        y^n &= (1-\beta_n)x^n+\beta_n Tx^n,\\
        x^{n+1}&=(1-\alpha_n)x^n+\alpha_n Ty^n,
    \end{array}
\end{equation}
where $\alpha^n,\beta^n$ satisfy 
\begin{equation}\label{alpha beta}
    0\le \alpha_n\le \beta_n<1, \ \lim\limits_{n\rightarrow \infty}\beta_n=0, \ \sum\limits_{n}\alpha_n\beta_n=\infty.
\end{equation}
\end{theorem}

Now we extend the existing PnP-GD, PnP-HQS, and PnP-FBS to the Ishikawa process.  
The gradient descent method aims to find $u$, such that 
\begin{equation}
    (I-D_\beta+\nabla G)(u)=0.
\end{equation}
Note that it can be transformed into the fixed point problem:
\begin{equation}
    u = D_\beta(u)-\nabla G(u).
\end{equation}
By letting $T=D_\beta-\nabla G$, we propose PnPI-GD, an abbreviation for PnP Ishikawa gradient descent as follows:
\begin{equation}\label{PnPI-GD}
    \begin{array}{ll}
        v^n & = (1-\beta_n)u^n+\beta_n( D_\beta(u^n)-\nabla G(u^n)),  \\
        u^{n+1} & = (1-\alpha_n)u^{n}+\alpha_n(D_\beta(v^n)-\nabla G(v^n)).
    \end{array}
\end{equation}
We summarize the algorithm in Algorithm \ref{alg PnPI-GD}.  Theorem \ref{thm 1} gives the global convergence of PnPI-GD. \begin{theorem}\label{thm 1}(proof in Appendix \ref{appendix theorem1})
$K$ is a compact convex set in $V$. Let $D_\beta:K\rightarrow K$ be Lipschitz pseudo-contractive, $G:K\rightarrow K$ be differentiable, proper, closed, and convex, with Lipschitz gradient $\nabla G$. $\{\alpha_n\},\{\beta_n\}$ be two sequences satisfying (\ref{alpha beta}). Assume that $Fix(D_\beta-\nabla G)\neq \emptyset $. Then, $u^n$ generated by PnPI-GD in Algorithm \ref{alg PnPI-GD} converges strongly to a fixed point in $Fix(D_\beta-\nabla G)$.
\end{theorem}

\begin{algorithm}
\caption{PnPI-GD}
\label{alg PnPI-GD}
\begin{algorithmic}
\STATE{Given $D_\beta,\{\alpha_n\},\{\beta_n\}, u^0, N$, set $n=0$}
\WHILE{$n<N$}
\STATE{$n=n+1$}
\STATE{$v^n=(1-\beta_n)u^n+\beta_n(D_\beta(u^n)-\nabla G(u^n))$}
\STATE{$u^n=(1-\alpha_n)u^n+\alpha_n(D_\beta(v^n)-\nabla G(v^n))$}
\ENDWHILE
\RETURN $u^N$
\end{algorithmic}
\end{algorithm}

\begin{algorithm}
\caption{PnPI-HQS}
\label{alg PnPI-HQS}
\begin{algorithmic}
\STATE{Given $D_\beta, \{\alpha_n\},\{\beta_n\}, u^0, N$, set $n=0$}
\WHILE{$n<N$}
\STATE{$n=n+1$}
\STATE{$x^n=Prox_{\frac{G}{\beta}}(u^n)$}
\STATE{$v^n=(1-\beta_n)u^n+\beta_nD_\beta(x^n)$}
\STATE{$y^n=Prox_{\frac{G}{\beta}}(v^n)$}
\STATE{$u^{n+1} = (1-\alpha_n)u^n+\alpha_nD_\beta(y^n)$}
\ENDWHILE
\RETURN $u^N$
\end{algorithmic}
\end{algorithm}

\begin{algorithm}[h]
\caption{PnPI-FBS}
\label{alg PnPI-FBS}
\begin{algorithmic}
\STATE{Given $D_\beta, \{\alpha_n\},\{\beta_n\}, \lambda u^0, N$, set $n=0$}
\WHILE{$n<N$}
\STATE{$n=n+1$}
\STATE{$x^n=u^n-\lambda \nabla G (u^n)$}
\STATE{$v^n=(1-\beta_n)u^n+\beta_nD_\beta(x^n)$}
\STATE{$y^n=v^n-\lambda \nabla G (v^n)$}
\STATE{$u^{n+1} = (1-\alpha_n)u^n+\alpha_nD_\beta(y^n)$}
\ENDWHILE
\RETURN $u^N$
\end{algorithmic}
\end{algorithm}

In PnPI-GD, $T$ takes the form of $D_\beta-\nabla G$. If $T$ takes the form of PnP-HQS as in (\ref{PnP-HQS}), $T=D_\beta\circ Prox_{\frac{G}{\beta}}$, we arrive at PnPI-HQS:
\begin{equation}\label{PnPI-HQS}
    \begin{array}{ll}
        v^n & = (1-\beta_n)u^n+\beta_nD_\beta\circ Prox_{\frac{G}{\beta}}u^n,  \\
        u^{n+1} & = (1-\alpha_n)u^n+\alpha_nD_\beta\circ Prox_{\frac{G}{\beta}}v^n.
    \end{array}
\end{equation}
When $T$ is the Forward-Backward Splitting (FBS) operator, that is $T=D_\beta\circ(I-\lambda \nabla G)$, we arrive at PnPI-FBS: 
\begin{equation}\label{PnPI-FBS}
    \begin{array}{ll}
        v^n & = (1-\beta_n)u^n+\beta_nD_\beta\circ(I-\lambda \nabla G)u^n,  \\
        u^{n+1} & = (1-\alpha_n)u^n+\alpha_nD_\beta\circ(I-\lambda \nabla G)v^n.
    \end{array}
\end{equation}
We summarize PnPI-HQS and PnPI-FBS in Algorithms \ref{alg PnPI-HQS}-\ref{alg PnPI-FBS}. The corresponding convergence results are given in Theorem \ref{thm 2} and Theorem \ref{thm 3}, respectively.\par

\begin{theorem}\label{thm 2}(proof in Appendix \ref{appendix theorem2})
$K$ is a compact convex set in $V$. Let $D_\beta: K\rightarrow K$ be Lipschitz $k$-strictly pseudo-contractive, $G:K\rightarrow K$ be proper, closed, and convex, $\nabla G$ is $\gamma$-cocoercive. $\{\alpha_n\},\{\beta_n\}$ be two sequences satisfying (\ref{alpha beta}). Assume that $Fix(D_\beta\circ Prox_{\frac{G}{\beta}})\neq \emptyset $. Then, $u^n$ generated by PnPI-HQS in Algorithm \ref{alg PnPI-HQS} converges strongly to a fixed point in $Fix(D_\beta\circ Prox_{\frac{G}{\beta}})$, if $k\le \frac{2\gamma+1}{2\gamma+2}.$
\end{theorem}

\begin{theorem}\label{thm 3}(proof in Appendix \ref{appendix theorem3})
$K$ is a compact convex set in $V$. Let $D_\beta: K\rightarrow K$ be Lipschitz $k$-strictly pseudo-contractive, $G:K\rightarrow K$ be proper, closed, and convex, $\nabla G$ is $\gamma$-cocoercive. $\{\alpha_n\},\{\beta_n\}$ be two sequences satisfying (\ref{alpha beta}). Assume that $Fix(D_\beta\circ(I-\lambda \nabla G))\neq \emptyset $. Then, $u^n$ generated by PnPI-FBS in Algorithm \ref{alg PnPI-FBS} converges strongly to a fixed point in $Fix(D_\beta\circ(I-\lambda \nabla G))$, if $0\le \lambda\le 2\gamma$, and $k\le 1-\frac{\lambda}{2\gamma}$.
\end{theorem}

\begin{remark}\label{remark1}
For a proper, closed, convex, and differentiable $G$, $\nabla G$ is $0$-cocoercive. As a result, according to Theorem \ref{thm 2}, PnPI-HQS converges, if $k\le \frac{1}{2}$. Similarly, when we select $\lambda \in [0,\gamma]$ in PnPI-FBS, we have $k=\frac{1}{2}\le 1-\frac{\lambda}{2\gamma}$. That is, a $\frac{1}{2}$-strictly pseudo-contractive denoiser $D_\beta$ satisfies the conditions in Theorems \ref{thm 2}-\ref{thm 3}. 
\end{remark}

\section{Training strategy}\label{sec 4}
In this section, we propose an effective training strategy to ensure that the denoiser is pseudo-contractive. Let $p$ be the distribution of the training set of clean images, and $[\sigma_{\min},\sigma_{\max}]$ be the interval of the noise level. We use $\sqrt{\frac{1}{\beta}}$ to represent the denoising strength as suggested in \cite{chan2016plug}. For example, if $\beta = \frac{1}{15^2}$, it means that $D_\beta$ removes Gaussian noises with standard derivation $\sigma=15$.

Let $p$ be the distribution of the training set of clean images, and $[\sigma_{\min},\sigma_{\max}]$ be the interval of the noise level. We use $\sqrt{\frac{1}{\beta}}$ to represent the denoising strength as suggested in \cite{chan2016plug}. For example, if $\beta = \frac{1}{15^2}$, it means that $D_\beta$ is for removing Gaussian denoiser with standard derivation $\sigma=15$. \par
In order to ensure the denoisers to be $k$-strictly pseudo-contractive, we need $\|kI+(1-k)J\|_*\le 1$. Let $\theta$ be the parameters of the denoiser $D_\beta$. An optimal $\hat{\theta}$ is a solution to the following problem:
\begin{equation}\label{tmp loss 1}
\begin{array}{ll}
    \arg\min_{\theta}\mathbb{E}_{x\sim p, \sqrt{\frac{1}{\beta}}\sim U[\sigma_{\min},\sigma_{\max}], \xi_\beta\sim \mathcal{N}(0,\frac{1}{\beta})}\|D_\beta(x+\xi_\beta;\theta)-x\|^2,\\
    
     s.t. \ \ \|kI+(1-k)J\|_*\le 1.
\end{array}
\end{equation}

\begin{algorithm}
\caption{Power iterative method}
\label{alg power}
\begin{algorithmic}
\STATE{Given $q^0$ with $\|q^0\|=1, J,N$, $n=0$}
\WHILE{$n<N$}
\STATE{$n=n+1$}
\STATE{$z^n=Jq^{n-1}$}
\STATE{$q^n=\frac{z^n}{\|z^n\|}$}
\ENDWHILE
\RETURN $\lambda^N=(q^N)^\mathrm{T}Jq^N$
\end{algorithmic}
\end{algorithm}

By the power iterative method \cite{golub2013matrix}, we compute the spectral norm of $J$ as in Algorithm \ref{alg power}. The AutoGrad toolbox in Pytorch \cite{paszke2017automatic} allows the calculation for $Jx$ and $J^\mathrm{T}x$ with any vector $x$. Thus, $z^n$ and $\lambda^N$ in Algorithm \ref{alg power} can be obtained efficiently. The constrained optimization problem (\ref{tmp loss 1}) can be solved by minimizing the following unconstrained loss function:
\begin{equation}\label{spc loss}
\begin{array}{ll}
    \mathbb{E}_{x\sim p, \sqrt{\frac{1}{\beta}}\sim U[\sigma_{\min},\sigma_{\max}], \xi_\beta\sim \mathcal{N}(0,\frac{1}{\beta})} \\
    \ \ \ \|D_\beta(x+\xi_\beta;\theta)-x\|^2+r\max\left\{\|kI+(1-k)J\|_*,1-\epsilon\right\},
\end{array}
\end{equation}
where $r>0$ is the balancing parameter, and $\epsilon\in(0,1)$ is a parameter that controls the constraint. 

In order to train a pseudo-contractive denoiser, we need to constrain 
\begin{equation}
    \langle (I-S)(x-y),x-y\rangle \ge 0, \forall x,y\in V,
\end{equation}
where $S=\frac{1}{2}(J+J^\mathrm{T})$ is the symmetric part of $J$. Since $S$ is symmetric, we can do functional calculi on $S$. We wish to find a holomorphic function $f:\mathbb{C}\rightarrow\mathbb{C}$ defined on the neighborhood of $\sigma(S)$, such that $f(S)$ is defined, and 
\begin{equation}
    f(\{z\in\mathbb{C}: real(z)\le 1\})=\{z\in \mathbb{C}: |z|\le 1\}.
\end{equation}
Then, by the spectral mapping theorem \cite{harte1972spectral,haase2005spectral}, there holds 
\begin{equation}
    \sigma(f(S))=f(\sigma(S)).
\end{equation}
We choose the following function 
\begin{equation}
    f(z)=\frac{z}{z-2}, \forall z\neq 2.
\end{equation}
$f$ is holomorphic on the neighborhood of the spectrum set of a pseudo-contractive denoiser. Besides, $f$ maps the half plane $\{z\in\mathbb{C}: real(z)\le 1\}$ to the unit disk $\{z\in \mathbb{C}: |z|\le 1\}$. 

If we need $\sigma(S)\subset\{z\in\mathbb{C}:real(z)\le 1\}$, we only need to constrain $\sigma(f(S))\subset\{z\in\mathbb{C}:|z|\le 1\}$. 
Note that the spectral radius of $S$ is always no larger than the spectral norm of $S$, $\rho(S)\le \|S\|_*$. As a result, we only need to constrain $\|f(S)\|_*\le 1$, because
\begin{equation}
\begin{array}{ll}
&\|f(S)\|_*\le 1\Rightarrow \rho(f(S))\le 1 \\
\Leftrightarrow& f(\sigma(S))=\sigma(f(S))\subset\{z\in\mathbb{C}:|z|\le 1\}\\
\Leftrightarrow&\sigma(S)\subset f^{-1}(\{z\in\mathbb{C}:|z|\le 1\})=\{z\in\mathbb{C}:real(z)\le 1\}\\
\Rightarrow & \text{ The regularity condition (\ref{reg PC}) holds.} \\
\Rightarrow & D_\beta \text{ is pseudo-contractive.}
\end{array}
\end{equation}
Therefore, the loss function for a pseudo-contractive denoiser is 
\begin{equation}\label{pc loss}
\begin{array}{ll}
    \mathbb{E}_{x\sim p, \sqrt{\frac{1}{\beta}}\sim U[\sigma_{\min},\sigma_{\max}], \xi_\beta\sim \mathcal{N}(0,\frac{1}{\beta})}\\
    \ \ \ \|D_\beta(x+\xi_\beta;\theta)-x\|^2+r\max\left\{\left\|(S-2I)^{-1}S\right\|_*,1-\epsilon\right\},
\end{array}
\end{equation}
where 
\begin{equation}
    S=\frac{(J^\mathrm{T}+J)(x+\xi_\beta;\theta)}{2}.
\end{equation}

According to the power iterative method in Algorithm \ref{alg power}, in order to evaluate $\|S(S-2I)^{-1}\|_*$, given $q^{n-1}$, we need to calculate $z^n$, such that  
\begin{equation}
    z^n=(S-2I)^{-1}Sq^{n-1}, 
\end{equation}
which is the solution to the following least square problem:
\begin{equation}\label{least square}
    z^n=\arg\min\limits_{z}\frac{1}{2}\| (S-2I)z-Sq^{n-1}\|^2.
\end{equation}
We apply gradient descent to solve (\ref{least square}):
\begin{equation}
    z^n_{k+1}=z^n_{k}-dt (S-2I)[(S^\mathrm{T}-2I)z^n_k-S^\mathrm{T}q^{n-1}], \ k=1,2,3,...,K,
\end{equation}
where $z^n_1=z^{n-1}$, $z^n=z^n_{K+1}$. Besides, by substituting $z^N=(S-2I)^{-1}Sq^{N-1}$, we have 
\begin{equation}
    \lambda^N=(q^N)^\mathrm{T}(S-2I)^{-1}Sq^N=(q^N)^\mathrm{T}z^{N+1}=\langle q^N,z^{N+1}\rangle.
\end{equation}
We summarize this modified power iterative method in Algorithm \ref{alg MPIM}. Algorithm \ref{alg MPIM} extends the existing Algorithm \ref{alg power} to evaluate the spectral norm of the multiplication of an inverse matrix $(S-2I)^{-1}$ and $S$. By Algorithm \ref{alg MPIM}, we are able to minimize the loss in (\ref{pc loss}). 

\begin{algorithm}[htbp]
\caption{Modified power iterative method}
\label{alg MPIM}
\begin{algorithmic}
\STATE{Given $q^0$ with $\|q^0\|=1, S,N,K,dt$, $n=0$}
\WHILE{$n\le N$}
\STATE{$n=n+1$}
\STATE{$k=0$}
\STATE{$z_1^n=z^{n-1}$}
\WHILE{$k < K$}
\STATE{$k=k+1$}
\STATE{$z^n_{k+1}=z^n_k-dt(S-2I)[(S^\mathrm{T}-2I)z^n_k-S^\mathrm{T}q^{n-1}]$}
\ENDWHILE
\STATE{$z^n=z^n_{K+1}$}
\STATE{$q^n=\frac{z^n}{\|z^n\|}$}
\ENDWHILE
\RETURN $\lambda^N=\langle q^N,z^{N+1}\rangle$
\end{algorithmic}
\end{algorithm}

\section{Experiments}\label{sec 5}
In this section, we learn pseudo-contractive denoiser and $k$-strictly according to (\ref{pc loss}) and (\ref{spc loss}) with $k=\frac{1}{2}$. According to Theorems \ref{thm 2}-\ref{thm 3}, PnPI-HQS and PnPI-FBS converges with $\frac{1}{2}$-strictly pseudo-contractive denoisers. We use Set12 \cite{zhang2017beyond} as a test set to show the effectiveness of our method. All the experiments are conducted under Linux system, Python 3.8.12 and Pytorch 1.10.2. \par

\subsection{Training details}
For the pseudo-contractive Gaussian denoisers, we select DRUNet \cite{zhang2021plug}, which combines a U-Net and residual blocks. DRUNet takes the noisy image, as well as the noise level $\sigma$ as input, which is convenient for PnP image restoration.


For training details, we collect 800 images from DIV2K \cite{Ignatov_2018_ECCV_Workshops} as the training set and crop them into small patches of size $64\times 64$. The batch size is $32$. We add the Gaussian noise with $[\sigma_{\min},\sigma_{\max}]=[0,60]$ to the clean image. Adam optimizer is applied to train the model with learning rate $lr=10^{-4}$. For the parameters in Algorithms \ref{alg power}-\ref{alg MPIM}, we select $N=10, K=10, dt=0.1$, $\epsilon=0.1$, and $r=10^{-3}$ to ensure the regularity conditions in (\ref{rc spc}) and (\ref{rc pc}). \par

\subsection{Denoising performance}
We evaluate the Gaussian denoising performances of the proposed pseudo-contractive DRUNet (PC-DRUNet), $\frac{1}{2}$-strictly pseudo-contractive DRUNet (SPC-DRUNet), the non-expansive DRUNet (NE-DRUNet) trained with the loss (\ref{tmp loss 1}) with $k=0$, maximally monotone operator (MMO) in \cite{pesquet2021learning} which is firmly non-expansive, Prox-DRUNet in \cite{hurault2022proximal} with a contractive residual part, the standard DRUNet without extra regularizations, the classical FFDNet \cite{zhang2018ffdnet} and DnCNN \cite{zhang2017beyond}. For a fair comparison, all denoisers are trained with DIV2K, and the patch sizes are set to $64$. The PSNR values are given in Table \ref{denoising performance} on Set12.\par
\begin{table}[htbp]
\caption{Average denoising PSNR performance of different denoisers on Set12 dataset, for various
noise levels $\sigma$.}\label{denoising performance}
\begin{center}
\begin{tabular}{cccc}
\toprule
$\sigma$    & 15             & 25             & 40               \\
\midrule
FFDNet                &     32.08           & 29.99               & 27.90                  \\
DnCNN                 &32.88                &   30.46             & 28.26                 \\
DRUNet                & \textbf{33.08}          & \textbf{30.80}          & \textbf{28.76}            \\
\midrule
MMO & 31.36          & 29.06          & 27.00               \\
NE-DRUNet             & 31.68          & 29.57          & 27.18            \\
Prox-DRUNet           & 31.71          & 29.04          & 26.45           \\
SPC-DRUNet            & 32.90           & 30.59          & 28.44            \\
PC-DRUNet             & \textbf{33.01} & \textbf{30.69} & \textbf{28.66}  \\
\bottomrule
\end{tabular}
\end{center}
\end{table}
As shown in Table \ref{denoising performance}, restrictive conditions on the denoisers result in a compromised denoising performance. It can be explained by the spectrum distributions shown in Fig. \ref{spectral fig}: (a) for MMO; (b) for NE-DRUNet; (c) for Prox-DRUNet; (d) for SPC-DRUNet; (e) for PC-DRUNet; the complex plane $\mathbb{C}$ for DRUNet. A larger region means less restrictions on the Jacobian, and therefore, the denoising performance becomes better. In Fig. \ref{spectral fig}, we have $(a)\subset(b)\subset (d)\subset(e)\subset \mathbb{C}$, and the PSNR values in Table \ref{denoising performance} by MMO, NE-DRUNet, SPC-DRUNet, PC-DRUNet, and DRUNet have the same order. Note that when $\sigma=40$, Prox-DRUNet has poor performance, which means that the contractive residual assumption is harmful for large Gaussian noise. However, in Table \ref{denoising performance}, the proposed SPC-DRUNet and PC-DRUNet have better PSNR values, which indicates that the pseudo-contractiveness is a weaker and less harmful assumption on the deep denoisers.

\subsection{Assumption validations}

In the experiments, the strictly pseudo-contractive and pseudo-contractive conditions are softly constrained by the loss functions (\ref{spc loss}) and (\ref{pc loss}) with a trade-off parameter $r$. We validate the conditions in Table \ref{tab validation}. As shown in Table \ref{tab validation}, DRUNet without spectral regularization term is neither non-expansive nor pseudo-contractive. The norms $\|\frac{1}{2}I+\frac{1}{2}J\|_*$ and $\|(S-2I)^{-1}S\|_*$ are smaller than $1$, when PC-DRUNet and SPC-DRUNet are trained by the loss functions (\ref{spc loss}) and (\ref{pc loss}) with $r=10^{-3}$. It validates the effectiveness of the proposed training strategy.

\begin{table}[htbp]
\caption{Maximal values of different norms on Set12 dataset for various noise levels $\sigma$. }\label{tab validation}
\begin{center}
\begin{tabular}{ccccc}
\toprule
$\sigma $              & 15     & 25     & 40     &   Norm  \\
\midrule
DRUNet              & 2.1338 & 3.9356 & 6.0703 & $\|J\|_*$  \\
DRUNet              & 1.6277 & 2.7145 & 3.5150 & $\|\frac{1}{2}I+\frac{1}{2}J\|_*$ \\
DRUNet              & 5.4364 & 2.1429 & 2.2191 & $\|(S-2I)^{-1}S\|_*$  \\
\midrule
PC-DRUNet ($r=10^{-3}$)  & 0.9900 & 0.9933 & 0.9977 &   $\|\frac{1}{2}I+\frac{1}{2}J\|_*$  \\
PC-DRUNet ($r=10^{-4}$)  & 1.0010 & 1.2960 & 1.4646 &  $\|\frac{1}{2}I+\frac{1}{2}J\|_*$   \\
\midrule
SPC-DRUNet ($r=10^{-3}$) & 0.9938 & 0.9987 & 0.9981 &  $\|(S-2I)^{-1}S\|_*$   \\
SPC-DRUNet ($r=10^{-4}$) & 0.9996 & 1.1495 & 1.0826 &  $\|(S-2I)^{-1}S\|_*$ \\
\bottomrule
\end{tabular}
\end{center}
\end{table}

\subsection{PnP restoration}
In this section, we apply the proposed PnPI-GD, PnPI-HQS, and PnPI-FBS algorithms on deblurring, super-resolution, and poisson denoising experiments. For PnPI-GD, we choose the pretrained PC-DRUNet as $D_\beta$. For PnPI-HQS and PnPI-FBS, we choose the pretrained SPC-DRUNet with $k=\frac{1}{2}$. According to Theorem \ref{thm 2}, PnPI-HQS converges with $\frac{1}{2}$-strictly pseudo-contractive denoisers, because 

\begin{equation}
    \displaystyle k = \frac{1}{2}\le \frac{2\gamma+1}{2\gamma+2} 
\end{equation}
for any $\gamma\ge 0$. We apply a decreasing step size strategy in PnPI-HQS, by multiplying $\beta$ by the factor $1.01$ in each iteration. Similarly, by Theorem \ref{thm 3}, PnPI-FBS converges with $\frac{1}{2}$-strictly pseudo-contractive denoisers, if $\lambda\le\gamma$. \par
For the step size sequences $\{\alpha_n\},\{\beta_n\}$ in Algorithms \ref{alg PnPI-GD}-\ref{alg PnPI-FBS}, we let

\begin{equation}
    \alpha_n = \displaystyle\frac{1}{(n+1)^a}, \beta_n=\frac{1}{(n+1)^b}, n\ge 0
\end{equation}
with $ 0<b<a<1, a+b<1$, to satisfy the condition (\ref{alpha beta}) in Theorem \ref{thm 0.5}. Specifically, we let $a=0.3,b=0.15$ in PnPI-GD, and $a=0.8,b=0.15$ in PnPI-HQS and PnPI-FBS. In the deblurring and Poisson denoising tasks, the proposed methods are initialized with the observed image, that is $u^0=f$. In the single image super-resolution task, we choose $u^0$ as the bicubic interpolation of $f$ as in \cite{zhang2021plug}. \par
We compare our methods with some state-of-the-art convergent PnP methods including MMO-PnP-FBS \cite{pesquet2021learning}, which uses the FBS method with MMO denoiser; NE-PnP-PGD \cite{reehorst2018regularization,liu2021recovery} using PGD framework with NE-DRUNet; Prox-PnP-DRS \cite{hurault2022proximal}, which uses the Douglas-Rachfold Spilitting (DRS) method with Prox-DRUNet. We also indicate the results by DPIR \cite{zhang2021plug}, which applies PnP-HQS method with decreasing step size and DRUNet denoiser, but without convergence guarantee.

\subsubsection{Deblurring}
In the deblurring task, we seek to solve the inverse problem (\ref{original inverse problem}) with a convolution operator $K$ performed with circular boundary conditions and Gaussian noise $n$ with zero mean value and standard derivation $\sigma$. The fidelity term is
\begin{equation}
    G(u;f)=\frac{\mu}{2}\|Ku-f\|^2,
\end{equation}
where $\mu>0$ is the balancing parameter. The proximal operator $Prox_{\frac{G}{\beta}}$ can be efficiently calculated as in \cite{pan2016l_0}.
Note that in this case, $\nabla G(u) = \mu K^\mathrm{T}(Ku- f)$, and $\|K^\mathrm{T}K\|_*\le 1$. Therefore, $\nabla G$ is $\mu$-cocoercive. We set $N=300$, and fine tune $\mu$, $\lambda$ and $\beta>0$ in the proposed methods to achieve the best quantitive PSNR values. 

We demonstrate the effectiveness of our methods on the $8$ real-world camera shake kernels by Levin et al. \cite{levin2009understanding}, with $\sigma = 12.75$, and $17.85$ respectively. The kernels are shown in Fig. \ref{kernel}.

\begin{figure}[htbp]
\begin{minipage}{0.11\linewidth}
  \centerline{\includegraphics[width=1.6cm]{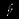}}
\end{minipage}
\hfill
\begin{minipage}{0.11\linewidth}
  \centerline{\includegraphics[width=1.6cm]{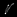}}
\end{minipage}
\hfill
\begin{minipage}{0.11\linewidth}
  \centerline{\includegraphics[width=1.6cm]{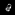}}
\end{minipage}
\hfill
\begin{minipage}{0.11\linewidth}
  \centerline{\includegraphics[width=1.6cm]{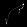}}
\end{minipage}
\hfill
\begin{minipage}{0.11\linewidth}
  \centerline{\includegraphics[width=1.6cm]{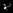}}
\end{minipage}
\hfill
\begin{minipage}{0.11\linewidth}
  \centerline{\includegraphics[width=1.6cm]{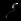}}
\end{minipage}
\hfill
\begin{minipage}{0.11\linewidth}
  \centerline{\includegraphics[width=1.6cm]{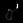}}
\end{minipage}
\hfill
\begin{minipage}{0.11\linewidth}
  \centerline{\includegraphics[width=1.6cm]{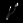}}
\end{minipage}
\centering\caption{Eight kernels from Levin et al. \cite{levin2009understanding}.}
\label{kernel}
\end{figure}

\begin{table}[htbp]
\caption{Average deblurring PSNR and SSIM performance by different methods on Set12 dataset with Levin's $8$ kernels with $\sigma = 12.75$.}\label{tab deblur5}
\begin{center}
\resizebox{1\textwidth}{!}{
\begin{tabular}{cccccccccc}
\toprule
 & kernel1         & kernel2         & kernel3         & kernel4         & {kernel5} & kernel6         & kernel7         & {kernel8} & Average         \\
 \midrule
MMO-PnP-FBS  & 25.61           & 25.44           & 26.39           & 25.22           & 27.24            & 26.72           & 26.45           & 25.69            & 26.10           \\
             & 0.7400          & 0.7363          & 0.7670          & 0.7251          & 0.7870           & 0.7706          & 0.7644          & 0.7430           & 0.7542          \\
NE-PnP-PGD   & 25.92           & 25.74           & 26.57           & 25.49           & 27.45            & 27.11           & 26.68           & 26.04            & 26.37           \\
             & {0.7516} & {0.7467} & {0.7738} & {0.7354} & {0.7974}  & {0.7843} & {0.7766} & 0.7550           & {0.7651} \\
Prox-PnP-DRS & {26.23}  & {26.02}  & {26.78}  & {25.76}  & {27.69}   & {27.31}  & {26.85}  & {25.74}   & {26.55}  \\
             &     0.7468            &  0.7433               & 0.7682                & 0.7314                & 0.7875                 &0.7757                 &0.7649                 &  0.7275                &   0.7557              \\
DPIR         & 27.37           & 27.08           & 27.71           & 26.93           & 28.56            & 28.24           & 27.90           & \textbf{27.31}   & 27.64           \\
             & 0.7895          & 0.7824          & 0.8010          & 0.7783          & 0.8231           & 0.8137          & 0.8084          & \textbf{0.7935}  & 0.7987          \\
PnPI-GD      & 26.58           & 26.36           & 26.96           & 26.10           & 27.85            & 27.54           & 27.31           & 26.58            & 26.91           \\
             & 0.7169          & 0.7249          & 0.7581          & 0.7143          & 0.7673           & 0.7526          & 0.7588          & 0.7342           & 0.7409          \\
PnPI-HQS     & \textbf{27.43}  & \textbf{27.14}  & \textbf{27.75}  & \textbf{26.95}  & \textbf{28.65}   & \textbf{28.33}  & \textbf{27.94}  & 27.29            & \textbf{27.69}  \\
             & \textbf{0.7936} & \textbf{0.7858} & \textbf{0.8048} & \textbf{0.7789} & \textbf{0.8300}  & \textbf{0.8188} & \textbf{0.8109} & 0.7930           & \textbf{0.8020} \\
PnPI-FBS     & 26.63           & 25.99           & 26.70           & 25.81           & 27.89            & 27.84           & 27.19           & 26.27            & 26.79           \\
             & 0.7727          & 0.7552          & 0.7789          & 0.7484          & 0.8121           & 0.8049          & 0.7936          & 0.7663           & 0.7790\\   
\bottomrule
\end{tabular}
}
\end{center}
\end{table}

\begin{table}[htbp]
\caption{Average deblurring PSNR and SSIM performance by different methods on Set12 dataset with Levin's $8$ kernels with $\sigma = 17.85$.}\label{tab deblur7}
\begin{center}
\resizebox{1\textwidth}{!}{
\begin{tabular}{cccccccccc}
\toprule
& kernel1         & kernel2         & kernel3         & kernel4         & kernel5         & kernel6         & kernel7         & kernel8         & Average         \\
\midrule
MMO-PnP-FBS  & 24.95           & 24.82           & 25.76           & 24.51           & 26.32           & 25.80           & 25.55           & 24.92           & 25.33           \\
             & 0.7137          & 0.7113          & 0.7433          & 0.6989          & 0.7514          & 0.7340          & 0.7281          & 0.7128          & 0.7242          \\
NE-PnP-PGD   & 25.18           & 25.03           & 25.89           & 24.72           & 26.49           & 26.12           & 25.76           & 25.22           & 25.55           \\
             & 0.7259          & 0.7218          & 0.7502          & 0.7101          & 0.7619          & 0.7505          & 0.7435          & 0.7266          & 0.7363          \\
Prox-PnP-DRS & 25.25           & 25.04           & 25.88           & 24.56           & 26.68           & 26.19           & 25.66           & 24.81           & 25.51           \\
             &   0.7153    & 0.7101       &    0.7403    &  0.6904      &  0.7545      &  0.7413     &   0.7154     &     0.6760   &  0.7179      \\
DPIR         & 26.17           & 25.97           & 26.66           & 25.73           & {27.33}  & 27.03           & 26.74           & \textbf{26.15}  & 26.48           \\
             & 0.7528          & 0.7483          & 0.7677          & 0.7412          & 0.7877          & 0.7776          & 0.7737          & 0.7566          & 0.7632          \\
PnPI-GD      & 25.61           & 25.46           & 26.13           & 25.10           & 26.83           & 26.54           & 26.18           & 25.54           & 25.92           \\
             & 0.6827          & 0.6940          & 0.7338          & 0.6801          & 0.7399          & 0.7241          & 0.7270          & 0.7093          & 0.7114          \\
PnPI-HQS     & \textbf{26.32}  & \textbf{26.07}  & \textbf{26.81}  & \textbf{25.73}  & \textbf{27.59}  & \textbf{27.14}  & \textbf{26.82}  & 26.07           & \textbf{26.57}  \\
             & \textbf{0.7629} & \textbf{0.7556} & \textbf{0.7796} & \textbf{0.7445} & \textbf{0.8037} & \textbf{0.7886} & \textbf{0.7819} & \textbf{0.7585} & \textbf{0.7719} \\
PnPI-FBS     & 25.85           & 25.37           & 26.01           & 25.07           & 27.03           & 26.80           & 26.38           & 25.62           & 26.01           \\
             & 0.7451          & 0.7328          & 0.7554          & 0.7216          & 0.7842          & 0.7733          & 0.7678          & 0.7439          & 0.7530 \\

\bottomrule
\end{tabular}
}
\end{center}
\end{table}

We summarize the PSNR and SSIM values with $\sigma = 12.75$ in Table \ref{tab deblur5}, and $\sigma = 17.85$ in Table \ref{tab deblur7}. The highest value is marked in \textbf{boldface}. It can be seen that in most cases, PnPI-HQS provides the best PSNR and SSIM values. Compared with the convergent PnP methods MMO-PnP-FBS, NE-PnP-PGD, and Prox-PnP-DRS, the proposed PnPI-GD and PnPI-FBS provides competitive results. It validates the effectiveness of PnP Ishikawa scheme and the pseudo-contractive denoisers. 

In Fig. \ref{485 blur}, we illustrate the results when recovering the image `Starfish' from kernel $8$ and Gaussian noise with $\sigma=12.75$. It can be seen from the enlarged parts that, PnPI-HQS provides the best visual result with sharp edges. Compared with MMO-PnP-FBS, NE-PnP-PGD, and Prox-PnP-DRS in Figs. \ref{485 blur} (b)-(d), results by PnPI-GD and PnPI-FBS have clearer structures.

\begin{figure}[htbp]

\begin{minipage}{0.19\linewidth}
  \centerline{\includegraphics[width=2.59cm]{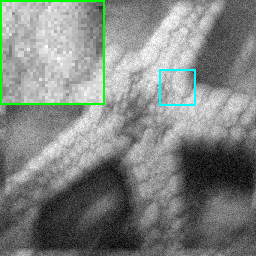}}
  \centerline{(a) Blurred}
\end{minipage}
\hfill
\begin{minipage}{0.19\linewidth}
  \centerline{\includegraphics[width=2.59cm]{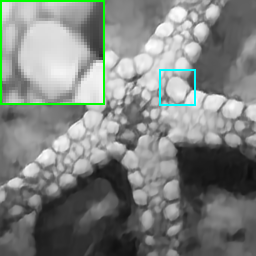}}
  \centerline{(b) MMO-FBS}
\end{minipage}
\hfill
\begin{minipage}{0.19\linewidth}
  \centerline{\includegraphics[width=2.59cm]{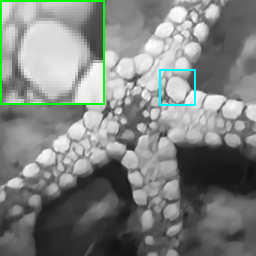}}
  \centerline{(c) NE-PGD}
\end{minipage}
\hfill
\begin{minipage}{0.19\linewidth}
  \centerline{\includegraphics[width=2.59cm]{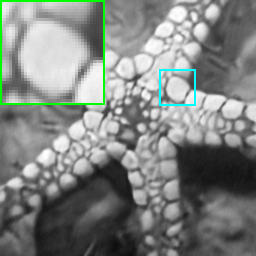}}
  \centerline{(d) Prox-DRS}
\end{minipage}
\hfill
\begin{minipage}{0.19\linewidth}
  \centerline{\includegraphics[width=2.59cm]{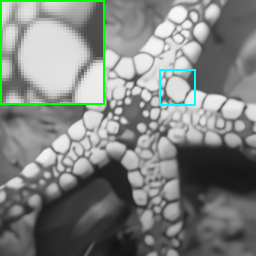}}
  \centerline{(e) DPIR}
\end{minipage}
\hfill
\begin{minipage}{0.19\linewidth}
  \centerline{\includegraphics[width=2.59cm]{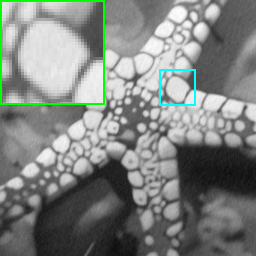}}
  \centerline{(f) PnPI-GD}
\end{minipage}
\hfill
\begin{minipage}{0.19\linewidth}
  \centerline{\includegraphics[width=2.59cm]{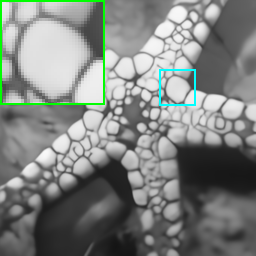}}
  \centerline{(g) PnPI-HQS}
\end{minipage}
\hfill
\begin{minipage}{0.19\linewidth}
  \centerline{\includegraphics[width=2.59cm]{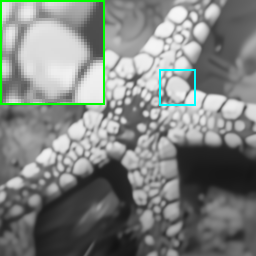}}
  \centerline{(h) PnPI-FBS}
\end{minipage}
\hfill
\begin{minipage}{0.39\linewidth}
  \centerline{\includegraphics[width=5.3cm]{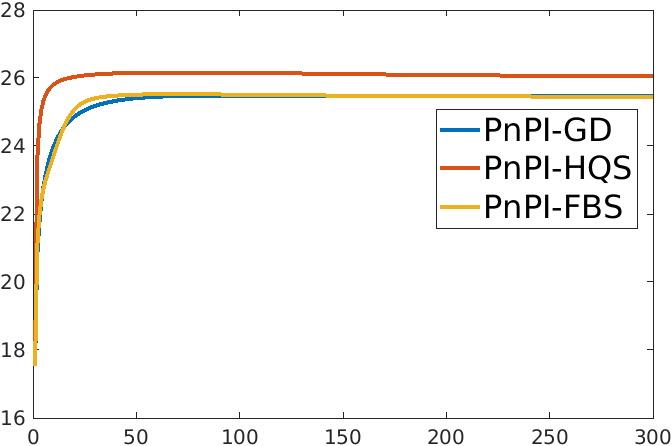}}
  \centerline{(i) }
\end{minipage}
\centering\caption{Results by different methods when recovering the image `Starfish' from kernel $8$ and Gaussian noise with $\sigma=12.75$. (a) Blurred. (b) MMO-FBS, PSNR=$24.64$dB. (c) NE-PGD, PSNR=$24.98$dB. (d) Prox-DRS, PSNR=$24.92$dB. (e) DPIR, PSNR=$25.92$dB. (f) PnPI-GD, PSNR=$25.49$dB. (g) PnPI-HQS, PSNR=$26.03$dB. (h) PnPI-FBS, PSNR=$25.46$dB. (i) PSNR curves by PnPI-GD, PnPI-HQS, and PnPI-FBS.}
\label{485 blur}
\end{figure}

\subsubsection{Single image super-resolution}

In the super-resolution task, we seek to solve the inverse problem (\ref{original inverse problem}) with a convolution operator $K$ performed with circular boundary conditions, a standard $s$-fold downsampling operator $S$, as well as Gaussian noise $n$ with zero mean value and standard derivation $\sigma$. The fidelity term is
\begin{equation}
    G(u;f)=\frac{\mu}{2}\|SKu-f\|^2,
\end{equation}
where $\mu>0$ is the balancing parameter. The proximal operator $Prox_{\frac{G}{\beta}}$ can be efficiently calculated as in \cite{zhang2021plug}.
Similarly to the deblurring task, $\nabla G(u) = \mu S^\mathrm{T}K^\mathrm{T}(Ku- f)$, and $\|S^\mathrm{T}K^\mathrm{T}KS\|_*\le 1$. Therefore, $\nabla G$ is $\mu$-cocoercive. We set $N=150$, and fine tune $\mu$, $\lambda$ and $\beta>0$ in the proposed methods to achieve the best quantitive PSNR values. 

We let the kernel $K$ be the isotropic Gaussian blur kernel with standard deviation $2$. The downsampling scale are set as $s=2,4$. The noise levels are set as $\sigma = 0,2.55,7.65$.

\begin{table}[htbp]
\caption{Average super-resolution PSNR and SSIM performance by different methods on Set12 dataset with different scales and noise levels.}\label{tab sisr}
\begin{center}
\resizebox{0.8\textwidth}{!}{
\begin{tabular}{ccccccc}
\toprule
             & \multicolumn{3}{c}{s=2}                             & \multicolumn{3}{c}{s=4}                             \\ 
             \midrule
         $\sigma$    & 0               & 2.55            & 7.65            & 0               & 2.55            & 7.65            \\
             \midrule
MMO-PnP-FBS  & 27.18           & 26.32           & 25.18           & 25.42           & 25.18           & 24.23           \\
             & 0.8247          & 0.7776          & 0.7162          & 0.7453          & 0.7341          & 0.6915          \\
NE-PnP-PGD   & 27.17           & 26.45           & 25.23           & 25.51           & 25.26           & 24.30           \\
             & 0.8242          & 0.7817          & 0.7315          & 0.7482          & 0.7370          & 0.6955          \\
Prox-PnP-DRS & 31.25           & 26.96           & 25.48           & 25.89           & 25.27           & 24.08           \\
             &      0.9108           &  0.7927               &    0.7366             &    0.7810             &      0.7425           &  0.6831               \\
DPIR         & 30.99           & 27.49           & 25.79           & \textbf{26.56}  & \textbf{25.94}  & 24.42           \\
             & 0.8976          & 0.8082          & 0.7458          & \textbf{0.7945} & \textbf{0.7648} & 0.6984          \\
PnPI-GD      & 27.13           & 26.59           & 25.54           & 25.95           & 25.57           & 24.42           \\
             & 0.8179          & 0.7919          & 0.7468          & 0.7735          & 0.7539          & 0.7021          \\
PnPI-HQS     & \textbf{31.87}  & \textbf{27.52}  & \textbf{25.98}  & 26.29           & 25.72           & \textbf{24.61}  \\
             & \textbf{0.9128} & \textbf{0.8115} & \textbf{0.7584} & 0.7905          & 0.7573          & \textbf{0.7090} \\
PnPI-FBS     & 28.96           & 26.50           & 25.59           & 25.99           & 25.66           & 24.49           \\
             & 0.8767          & 0.7891          & 0.7473          & 0.7678          & 0.7539          & 0.7048   \\  
             \bottomrule
\end{tabular}
}
\end{center}
\end{table}

We summarize the PSNR and SSIM values in Table \ref{tab sisr}. The highest value is marked in \textbf{boldface}. It can be seen that in most cases, PnPI-HQS provides the best PSNR and SSIM values. Compared with the convergent PnP methods MMO-PnP-FBS, NE-PnP-PGD, and Prox-PnP-DRS, the proposed PnPI-GD and PnPI-FBS provides competitive results, especially when the degradation is severe.

In Fig. \ref{fig barbara}, we illustrate the super-resolution results on the image `Barbara' with $s=2,\sigma=0$. It can be seen that, the result by PnPI-GD is smooth with clear edges, while PnPI-FBS recovers some textures. Note that in Figs. \ref{fig barbara} (d) and (g), Prox-PnP-DRS can recover some parts of textures of the tie, but PnPI-HQS can also recover the textures on the trouser.

\begin{figure}[htbp]

\begin{minipage}{0.19\linewidth}
  \centerline{\includegraphics[width=2.59cm]{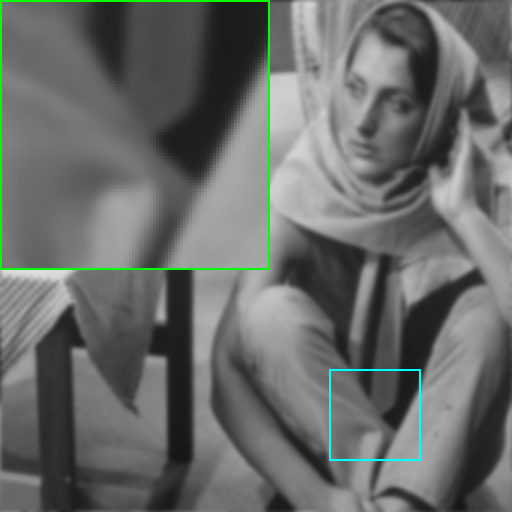}}
  \centerline{(a) LR}
\end{minipage}
\hfill
\begin{minipage}{0.19\linewidth}
  \centerline{\includegraphics[width=2.59cm]{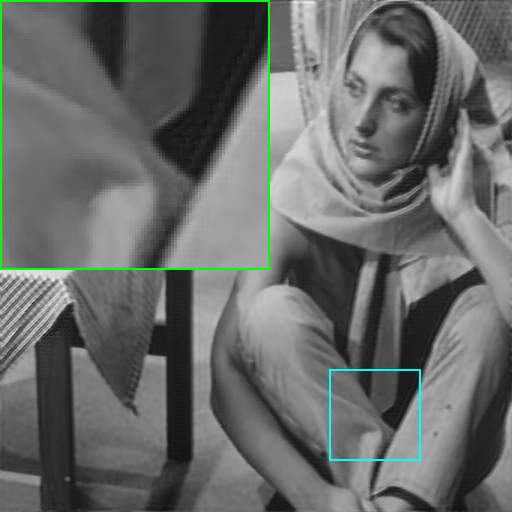}}
  \centerline{(b) MMO-FBS}
\end{minipage}
\hfill
\begin{minipage}{0.19\linewidth}
  \centerline{\includegraphics[width=2.59cm]{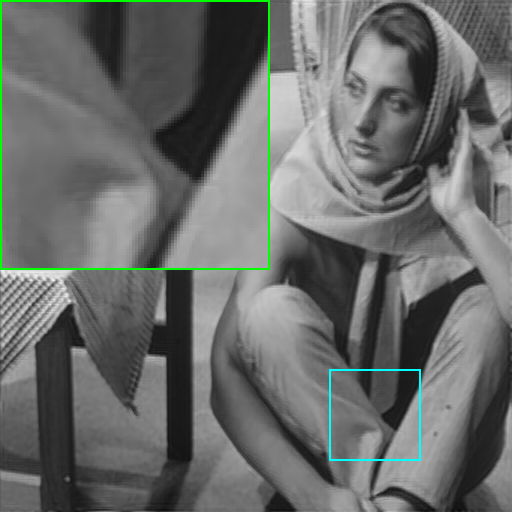}}
  \centerline{(c) NE-PGD}
\end{minipage}
\hfill
\begin{minipage}{0.19\linewidth}
  \centerline{\includegraphics[width=2.59cm]{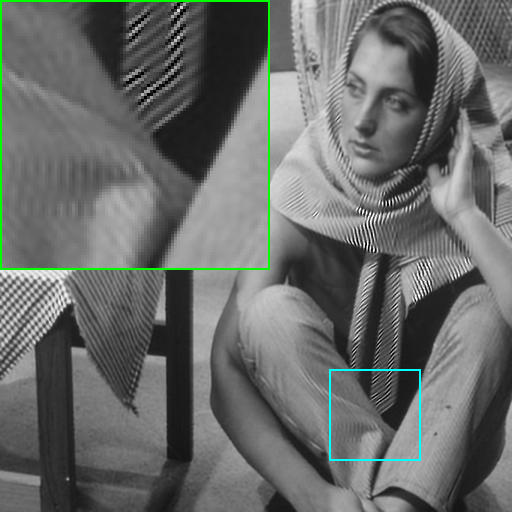}}
  \centerline{(d) Prox-DRS}
\end{minipage}
\hfill
\begin{minipage}{0.19\linewidth}
  \centerline{\includegraphics[width=2.59cm]{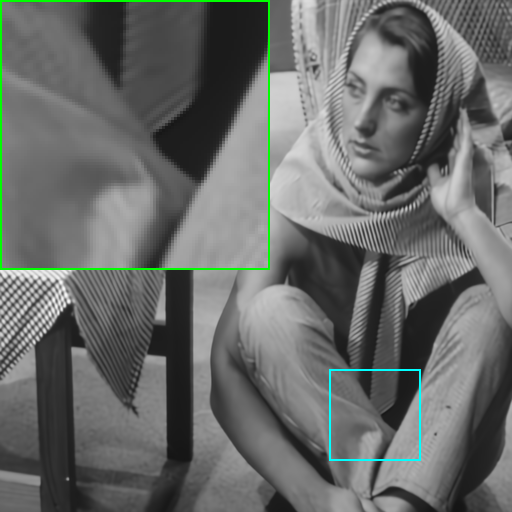}}
  \centerline{(e) DPIR}
\end{minipage}
\hfill
\begin{minipage}{0.19\linewidth}
  \centerline{\includegraphics[width=2.59cm]{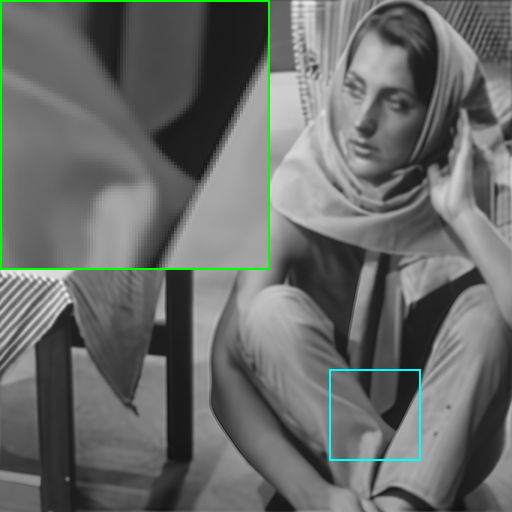}}
  \centerline{(f) PnPI-GD}
\end{minipage}
\hfill
\begin{minipage}{0.19\linewidth}
  \centerline{\includegraphics[width=2.59cm]{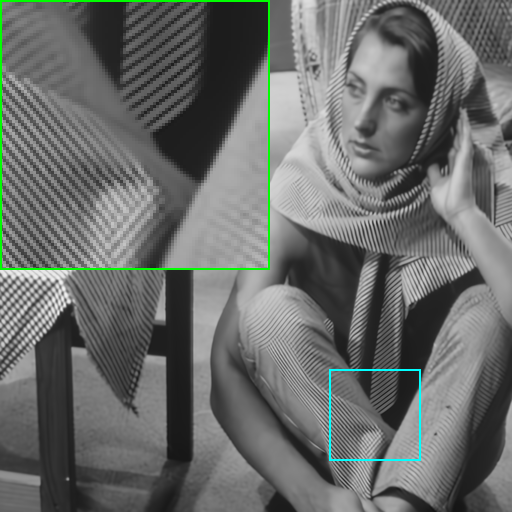}}
  \centerline{(g) PnPI-HQS}
\end{minipage}
\hfill
\begin{minipage}{0.19\linewidth}
  \centerline{\includegraphics[width=2.59cm]{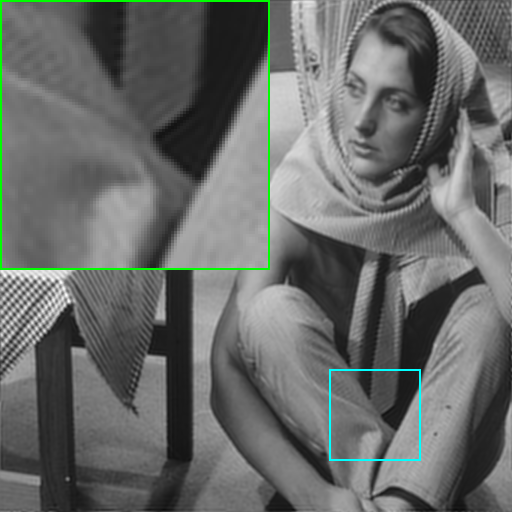}}
  \centerline{(h) PnPI-FBS}
\end{minipage}
\hfill
\begin{minipage}{0.39\linewidth}
  \centerline{\includegraphics[width=5.3cm]{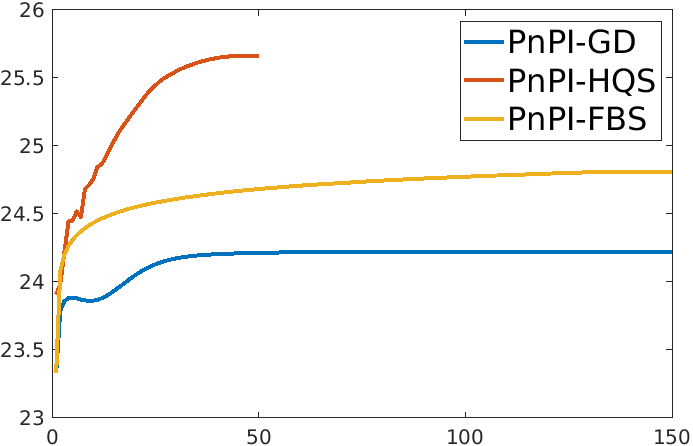}}
  \centerline{(i) }
\end{minipage}
\centering\caption{Super-resolution results by different methods on the image `Barbara' with $s=2,\sigma=0$. (a) Low-resolution (LR). (b) MMO-FBS, PSNR=$24.21$dB. (c) NE-PGD, PSNR=$24.24$dB. (d) Prox-DRS, PSNR=$24.81$dB. (e) DPIR, PSNR=$25.01$dB. (f) PnPI-GD, PSNR=$24.23$dB. (g) PnPI-HQS, PSNR=$25.66$dB. (h) PnPI-FBS, PSNR=$24.83$dB. (i) PSNR curves by PnPI-GD, PnPI-HQS, and PnPI-FBS.}
\label{fig barbara}
\end{figure}

\subsubsection{Poisson denoising}
In the Poisson noise removal task, we seek to solve the inverse problem (\ref{original inverse problem}) with $K$ be the identity operator and Poisson noise, that is 
\begin{equation}
    f\sim \displaystyle \frac{Poisson(u\times peak)}{peak}, 
\end{equation}
where $peak>0$ determines the noise level. A large $peak$ corresponds to a low noise level. Note that in this case, the gray value interval of $u$ is $[0,1]$. The fidelity term is 
\begin{equation}
    G(u;f) = \mu\langle u-f\log u,\textbf{1}\rangle.
\end{equation}
The proximal operator $Prox_{\frac{G}{\beta}}$ can be solved according to \cite{kumar2019low}. Since $\nabla G(u) = \textbf{1}-\frac{f}{u}$ is not cocoercive, gradient-based methods are not guaranteed to converge, such as MMO-PnP-FBS, NE-PnP-PGD, PnPI-GD, and PnPI-FBS. However, we still compare these methods when removing Poisson noises. Although PnPI-GD and PnPI-FBS are not guaranteed to converge, we observe in experiments that both algorithms converge efficiently. We set $N=100$, and fine tune $\mu,\lambda,\beta$. 

The overall PSNR and SSIM values are listed in Table \ref{tab poisson}. The highest value is marked in \textbf{boldface}. It can be seen that in most cases, PnPI-FBS provides the best PSNR values, while PnPI-HQS has the best SSIM values. Compared with the state-of-the-art PnP methods, the proposed PnPI-GD, PnPI-HQS, and PnPI-FBS provides competitive results.

In Fig. \ref{fig 20 poisson}, we show the Poisson noise removal results on the image `Lena' with $peak=20$. In Fig. \ref{fig 20 poisson} (a), the enlarged part is severely degraded. The methods in Figs. \ref{fig 20 poisson} (b)-(e) can recover some textures. Note that in Figs. \ref{fig 20 poisson} (f)-(h), the proposed methods restore finer textures, with less noise residuals.

\begin{table}[htbp]
\caption{Average Poisson denoising PSNR and SSIM performance by different methods on Set12 dataset with different peaks.}\label{tab poisson}

\begin{center}
\resizebox{0.8\textwidth}{!}{
\begin{tabular}{ccccccc}
\toprule
             & \multicolumn{2}{c}{peak=10}          & \multicolumn{2}{c}{peak=15}         & \multicolumn{2}{c}{peak=20}         \\
            \toprule
            & PSNR & SSIM & PSNR & SSIM & PSNR &SSIM\\
             \midrule
MMO-FBS  & 24.75   & 0.7102        & 25.98   & 0.7486        & 26.67   & 0.7642        \\
NE-PGD   & 25.57   & 0.7293        & 26.50    & 0.7634        & 27.20  & 0.7892       \\
Prox-DRS & 26.27    &     0.7468       & 26.71  &   0.7682         & 27.21    & 0.7834       \\
DPIR         & 25.97   & 0.7316      & 27.09 & 0.7923         & 27.87     & 0.8180    \\
PnPI-GD      & 25.90    & 0.7270        & 27.03       & 0.7567    & 27.53      & 0.7797      \\
PnPI-HQS     & 25.82   & \textbf{0.7796}       & 27.20   & \textbf{0.8043}         & 28.11  & \textbf{0.8212}         \\
PnPI-FBS     & \textbf{26.80} & 0.7689  & \textbf{27.58} & 0.7827 & \textbf{28.16} & 0.8010\\
             \bottomrule
\end{tabular}
}
\end{center}

\end{table}

\begin{figure}[htbp]

\begin{minipage}{0.19\linewidth}
  \centerline{\includegraphics[width=2.59cm]{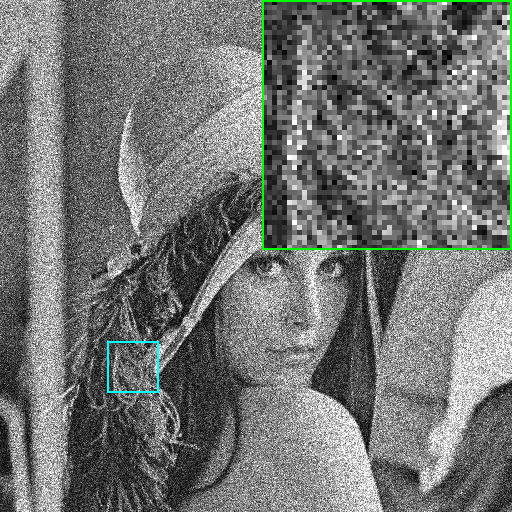}}
  \centerline{(a) Noisy}
\end{minipage}
\hfill
\begin{minipage}{0.19\linewidth}
  \centerline{\includegraphics[width=2.59cm]{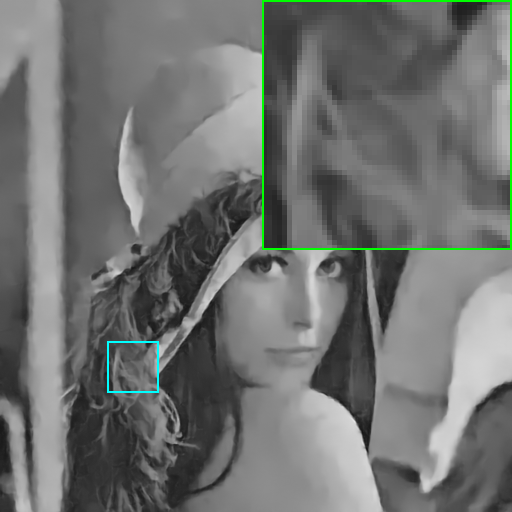}}
  \centerline{(b) MMO-FBS}
\end{minipage}
\hfill
\begin{minipage}{0.19\linewidth}
  \centerline{\includegraphics[width=2.59cm]{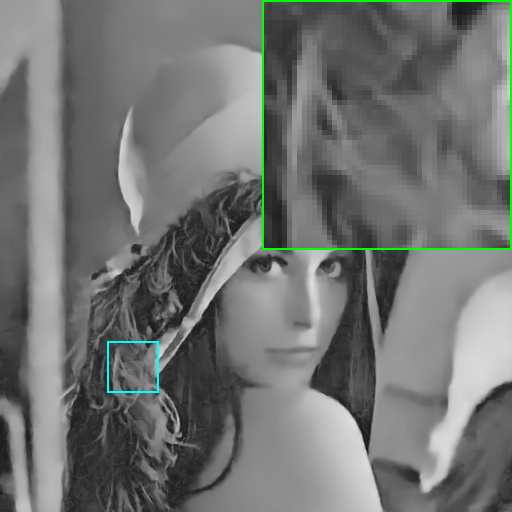}}
  \centerline{(c) NE-PGD}
\end{minipage}
\hfill
\begin{minipage}{0.19\linewidth}
  \centerline{\includegraphics[width=2.59cm]{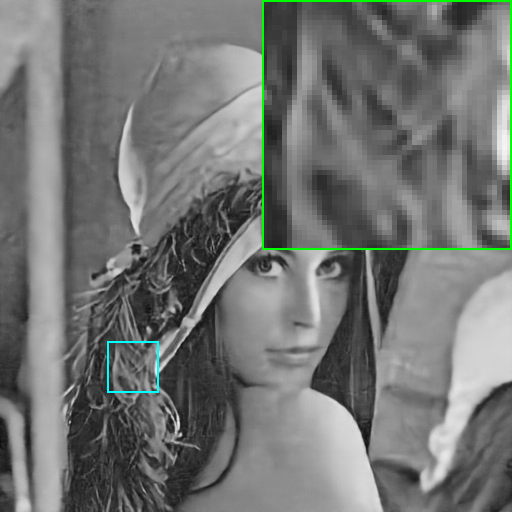}}
  \centerline{(d) Prox-DRS}
\end{minipage}
\hfill
\begin{minipage}{0.19\linewidth}
  \centerline{\includegraphics[width=2.59cm]{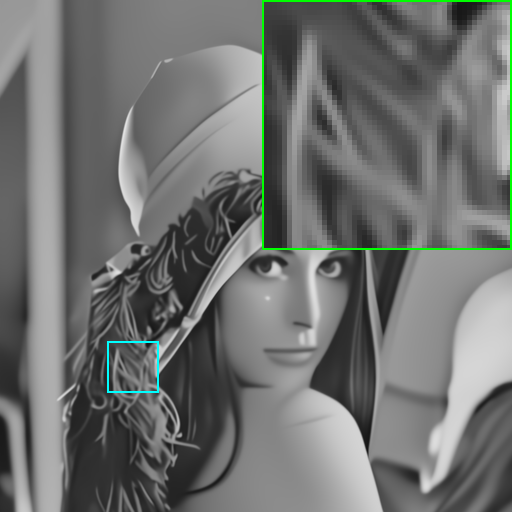}}
  \centerline{(e) DPIR}
\end{minipage}
\hfill
\begin{minipage}{0.19\linewidth}
  \centerline{\includegraphics[width=2.59cm]{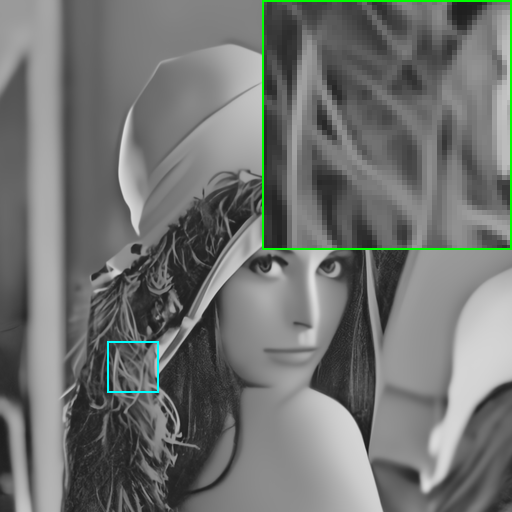}}
  \centerline{(f) PnPI-GD}
\end{minipage}
\hfill
\begin{minipage}{0.19\linewidth}
  \centerline{\includegraphics[width=2.59cm]{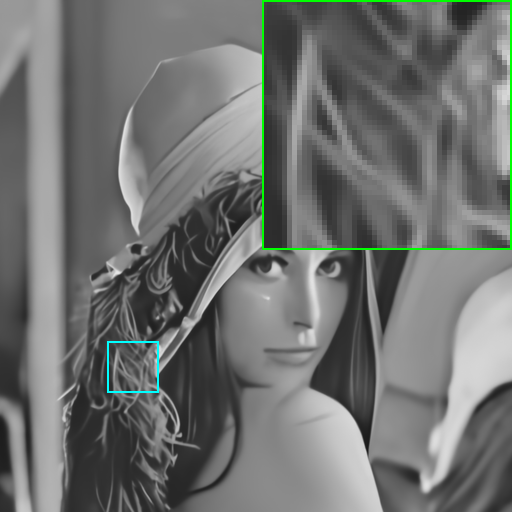}}
  \centerline{(g) PnPI-HQS}
\end{minipage}
\hfill
\begin{minipage}{0.19\linewidth}
  \centerline{\includegraphics[width=2.59cm]{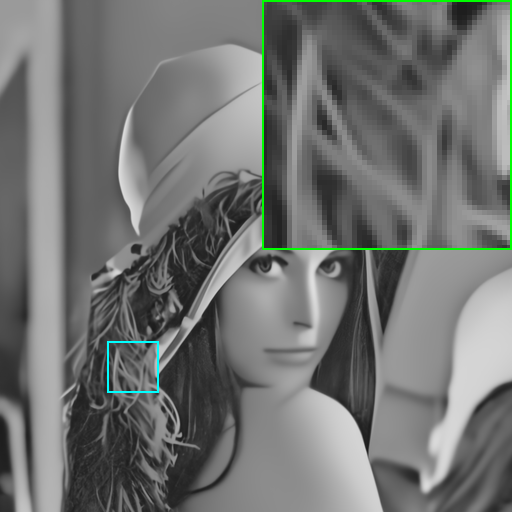}}
  \centerline{(h) PnPI-FBS}
\end{minipage}
\hfill
\begin{minipage}{0.39\linewidth}
  \centerline{\includegraphics[width=5.3cm]{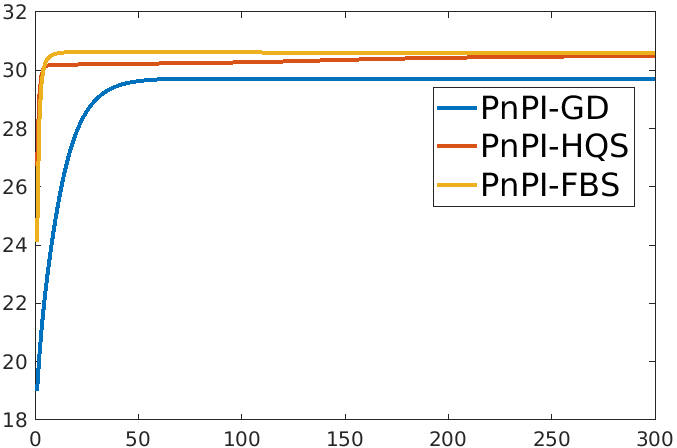}}
  \centerline{(i) }
\end{minipage}
\centering\caption{Results by different methods when recovering the image `Lena' from Poisson noise (peak=20). (a) Noisy. (b) MMO-FBS, PSNR=$29.39$dB. (c) NE-PGD, PSNR=$29.84$dB. (d) Prox-DRS, PSNR=$30.24$dB. (e) DPIR, PSNR=$30.14$dB. (f) PnPI-GD, PSNR=$29.73$dB. (g) PnPI-HQS, PSNR=$30.43$dB. (h) PnPI-FBS, PSNR=$30.82$dB. (i) PSNR curves by PnPI-GD, PnPI-HQS, and PnPI-FBS.}
\label{fig 20 poisson}
\end{figure}

\section{Conclusion}\label{sec 6}

This paper introduces a novel training strategy that enforces a weaker constraint on the deep denoiser called pseudo-contractiveness. By studying the spectrum of the Jacobian matrix, we uncover relationships between different denoiser assumptions. Utilizing the Ishikawa process, efficient fixed-point algorithms are derived. The proposed algorithms demonstrate strong theoretical convergence towards a fixed point. To enforce the pseudo-contractive denoiser assumption, a training strategy based on holomorphic transformation and functional calculi is proposed. Extensive experiments showcase the superior performance of the pseudo-contractive denoiser compared to other related denoisers, both visually and quantitatively. Overall, the proposed methods offer competitive results for image restoration tasks.

\section*{Appendix}\label{appendix}

\section*{Proofs to Lemmas and Theorems}
Before the proofs, we review Lemma \ref{lemma 5} from \cite{giselsson2017tight}.
\begin{lemma}\label{lemma 5}
Let $G$ be proper, closed, and convex, $\nabla G$ is $\gamma$-cocoercive, that is for any $x,y\in V$, there holds
\begin{equation}
    \langle x-y,\nabla G(x)-\nabla(y)\rangle \ge \gamma\|\nabla G(x)-\nabla G(y)\|^2.
\end{equation}
Then, the resolvent of $\nabla G$, which is the proximal operator $P=Prox_G=(I+\nabla G)^{-1}$, is $\frac{1}{2\gamma+2}$-averaged. 
\end{lemma}

\section*{Proof to Lemma \ref{lemma 1}}\label{appendix lemma1}
\begin{proof}
  By the definition, $D$ is said to be $k$-strictly pseudo-contractive with $k<1$, if $\forall x,y\in V$, we have
\begin{equation}
\|D(x)-D(y)\|^2\le \|x-y\|^2+k\|(I-D)(x)-(I-D)(y)\|^2.
\end{equation}
Denote $a = D(x)-D(y), b=x-y$. Then 
\begin{equation}
    \begin{array}{ll}
        \|a\|^2 &\le \|b\|^2+k\|a-b\|^2  = \|b\|^2+k\|a\|^2+k\|b\|^2-2k\langle a,b\rangle,\vspace{1ex}\\
         (1-k)\|a\|^2+2k\langle a,b\rangle&\le (1+k)\|b\|^2,\vspace{1ex}\\
\displaystyle         \|a\|^2+\frac{2k}{1-k}\langle a,b\rangle&\le \displaystyle \frac{1+k}{1-k}\|b\|^2,\vspace{1ex}\\
         \left\|a+\displaystyle\frac{k}{1-k}b\right\|^2&\le\left(\displaystyle\frac{k^2}{(1-k)^2}+\frac{1+k}{1-k}\right)\|b\|^2 =\left(\displaystyle\frac{1}{(1-k)^2}\right)\|b\|^2,\vspace{1ex}\\
         \|(1-k)a+kb\|^2&\le\|b\|,
    \end{array}
\end{equation}
which means that $(1-k)D+kI$ is non-expansive. Let $N=(1-k)D+kI$, we have
\begin{equation}
    D = \frac{1}{1-k}N-\frac{k}{1-k}I.
\end{equation}
\end{proof}

\section*{Proof to Lemma \ref{lemma 1.5}}\label{appendix lemma1.5}
\begin{proof}
Lemma \ref{lemma 1.5} is a straight conclusion of (\ref{tmp 1245}).
\end{proof}

\section*{Proof to Lemma \ref{lemma 6}}\label{appendix lemma6}
\begin{proof}
According to Lemma \ref{lemma 1.5}, we only need to show that for any $x,y\in V$, there holds
\begin{equation}
    \langle (I-T)(x)-(I-T)(y),x-y\rangle\ge 0.
\end{equation}
Note that $I-T=I-D+\nabla G$. Thus we have 
\begin{equation}
\begin{array}{ll}
&\langle (I-T)(x)-(I-T)(y),x-y\rangle\\
=&\langle (I-D)(x)-(I-D)(y),x-y\rangle + \langle \nabla G(x)-\nabla G(y),x-y\rangle \\
\ge & 0+0=0.
\end{array}
\end{equation}
The last $\ge$ comes from the pseudo-contractive $D$ and convex $G$.
\end{proof}

\section*{Proof to Lemma \ref{lemma 4}}\label{appendix lemma4}
Before the proof to Lemma \ref{lemma 4}, we give the following Lemma \ref{lemma 3}.
\begin{lemma}\label{lemma 3}
Let $V$ be the real Hilbert space. For any $x,y\in V$ and $\alpha,\beta\in \mathbb{R}$, there holds
\begin{equation}
\|\alpha x+\beta y\|^2=\alpha(\alpha+\beta)\|x\|^2+\beta(\alpha+\beta)\|y\|^2-\alpha\beta\|x-y\|^2,    
\end{equation}
and 
\begin{equation}
    \alpha\beta\|x+y\|^2=\alpha(\alpha+\beta)\|x\|^2+\beta(\alpha+\beta)\|y\|^2-\|\alpha x-\beta y\|^2.
\end{equation}
\end{lemma}

\begin{proof}
Here we only prove the first equality. By letting $x'=x,y'=-y$, the second equality holds naturally. The left hand side equals to
\begin{equation}
\begin{array}{ll}
    LHS=\alpha^2\|x\|^2+\beta^2\|y\|^2+2\alpha\beta\langle x,y\rangle,
\end{array}
\end{equation}
while the right hand side equals to
\begin{equation}
\begin{array}{ll}
    RHS&=\alpha^2\|x\|^2+\beta^2\|y\|^2+\alpha\beta(\|x\|^2+\|y\|^2)-\alpha\beta(\|x\|^2+\|y\|^2-2\langle x,y\rangle)\\
    &=\alpha^2\|x\|^2+\beta^2\|y\|^2+2\alpha\beta\langle x,y\rangle=LHS.
\end{array}
\end{equation}
\end{proof}

Now we can prove Lemma \ref{lemma 4}.
\begin{proof}
Since $D$ is $k$-strictly pseudo-contractive, and that $P$ is $\theta$-averaged, for any $x,y\in V$ we have 
    \begin{equation}
        \begin{array}{ll}
            &\|D\circ P(x)-D\circ P(y)\|^2\\
            \le & \|P(x)-P(y)\|^2+k\|(I-D)\circ P(x)-(I-D)\circ P(y)\|^2\\
            \le &\|x-y\|^2 - \frac{1-\theta}{\theta}\|(I-P) (x)-(I-P) (y)\|^2\\ &+k\|(I-D)\circ P(x)-(I-D)\circ P(y)\|^2.\\
        \end{array}
    \end{equation}
    Set
    \begin{equation}
        \alpha = -\frac{1-\theta}{\theta}, \beta = k, l=\displaystyle\frac{\alpha\beta}{\alpha+\beta}=\frac{k\frac{1-\theta}{\theta}}{\frac{1-\theta}{\theta}-k}=\frac{k(1-\theta)}{(1-\theta)-k\theta}.
    \end{equation}
By Lemma \ref{lemma 3}, there holds
    \begin{equation}
        \begin{array}{ll}
        &\alpha\|(I-P) (x)-(I-P) (y)\|^2+\beta\|(I-D)\circ P(x)-(I-D)\circ P(y)\|^2\vspace{0.5ex}\\
        =&\displaystyle\frac{\alpha\beta}{\alpha+\beta}\|[(I-P)+(I-D)\circ P](x)-[(I-P)+(I-D)\circ P](y)\|^2\vspace{0.5ex}\\
        &+\displaystyle\frac{1}{\alpha+\beta}\|\alpha(I-D)\circ P(x)-\alpha(I-D)\circ P(y) - \beta (I-P)(x)+\beta(I-P)(y)\|^2\vspace{0.5ex}\\
        =&\displaystyle\frac{\alpha\beta}{\alpha+\beta}\|(I-D\circ P)(x)-(I-D\circ P)(y)\|^2\vspace{0.5ex}\\
        &+\displaystyle\frac{1}{\alpha+\beta}\|\alpha(I-D)\circ P(x)-\alpha(I-D)\circ P(y) - \beta (I-P)(x)+\beta(I-P)(y)\|^2.
        \end{array}
    \end{equation}
When $k\le1-\theta$, $\alpha+\beta<0$, and thus
    \begin{equation}
        \begin{array}{ll}
            &\alpha\|(I-P) (x)-(I-P) (y)\|^2+\beta\|(I-D)\circ P(x)-(I-D)\circ P(y)\|^2\vspace{0.5ex}\\
          \le&\displaystyle\frac{\alpha\beta}{\alpha+\beta}\|(I-D\circ P)(x)-(I-D\circ P)(y)\|^2\\
            =&l\|(I-D\circ P)(x)-(I-D\circ P)(y)\|^2.
        \end{array}
    \end{equation}
If $k=1-\theta$, $l=1$. This completes the proof.
\end{proof}

\section*{Proof to Theorem \ref{thm 1}}\label{appendix theorem1}
\begin{proof}
By Lemma \ref{lemma 6}, $T=D_\beta-\nabla G$ is Lipschitz and pseudo-contractive. Therefore, according to Ishikawa's Theorem \cite{ishikawa1974fixed}, PnPI-GD converges strongly in $Fix(D_\beta-\nabla G)$.
\end{proof}

\section*{Proof to Theorem \ref{thm 2}}\label{appendix theorem2}
\begin{proof}
By Lemma \ref{lemma 5}, since $\nabla G$ is $\gamma$-cocoercive, the proximal operator $Prox_{\frac{G}{\beta}}$ is $\frac{1}{2\gamma+2}$-averaged. Since $k<\frac{2\gamma+1}{2\gamma +2}=1-\frac{1}{2\gamma+2}$, by Lemma \ref{lemma 4}, $T=D_\beta\circ Prox_{\frac{G}{\beta}}$ is $l$-strictly pseudo-contractive, where 
\begin{equation}
    0\le l=\displaystyle\frac{k(1-\frac{1}{2\gamma+2})}{(1-\frac{1}{2\gamma+2})-k\frac{1}{2\gamma+2}}=\frac{k(2\gamma+1)}{2\gamma+1-k}<1.
\end{equation}
$D_\beta$ is $k$-strictly pseudo-contractive, and thus, $D_\beta$ is $\frac{1+k}{1-k}$-Lipschitz. Since $Prox_{\frac{G}{\beta}}$ is $1$-Lipschitz, $T=D_\beta\circ Prox_{\frac{G}{\beta}}$ is also Lipschitz. Therefore, $T$ is Lipschitz and pseudo-contractive.
According to Ishikawa's Theorem \cite{ishikawa1974fixed}, when $Fix(D_\beta\circ Prox_{\frac{G}{\beta}})\neq\emptyset$, PnPI-HQS converges strongly in $Fix(D_\beta\circ Prox_{\frac{G}{\beta}})$.
\end{proof}

\section*{Proof to Theorem \ref{thm 3}}\label{appendix theorem3}
\begin{proof}
$\nabla G$ is $\gamma$-cocoercive. After some derivations, we have that for any $x,y\in V$, 
\begin{equation}
\begin{array}{ll}
         &\langle x-y,\nabla G(x)-\nabla G(y) \rangle-\gamma \|\nabla G(x)-\nabla G(y)\|^2 \\ 
         \iff & \|(2\gamma\nabla G-I)(x)-(2\gamma\nabla G-I)(y)\|^2 \le \|x-y\|^2.
\end{array}
\end{equation}
It means that $2\gamma\nabla G-I$ is non-expansive. For $0\le \lambda \le 2\gamma$, 
\begin{equation}
    I-\lambda \nabla G =\displaystyle\left(1-\frac{\lambda}{2\gamma}\right)I+\frac{\lambda}{2\gamma}(I-2\gamma\nabla G).
\end{equation}
Therefore, $I-\lambda\nabla G$ is $\frac{\lambda}{2\gamma}$-averaged. By Lemma \ref{lemma 5}, when $D_\beta$ is $k$-strictly pseudo-contractive, $D_\beta\circ (I-\lambda\nabla G)$ is $l$-strictly pseudo-contractive, where 
\begin{equation}
    0\le l=\displaystyle\frac{k(1-\frac{\lambda}{2\gamma})}{(1-\frac{\lambda}{2\gamma})-k\frac{\lambda}{2\gamma}}=\frac{k(2\gamma-\lambda)}{2\gamma-\lambda-k\lambda}<1,
\end{equation}
if $k\le 1-\frac{\lambda}{2\gamma}$. Since $D_\beta$ is $k$-strictly pseudo-contractive, it is $\frac{1+k}{1-k}$-Lipschitz. Then $T=D_\beta\circ (I-\lambda\nabla G)$ is Lipschitz. Under the assumption that $Fix(T)\neq \emptyset$, Ishikawa's Theorem \cite{ishikawa1974fixed} guarantees the strong convergence of PnPI-FBS in $Fix(T)$.

\end{proof}

\section*{Acknowledgments}
We would like to acknowledge the assistance of volunteers in putting
together this example manuscript and supplement.

\end{document}